\documentclass[10pt,journal]{IEEEtran}

\ifCLASSOPTIONcompsoc
  \usepackage[nocompress]{cite}
\else
  \usepackage{cite}
\fi
\ifCLASSINFOpdf
\else
\fi

\usepackage[compact]{titlesec}
\usepackage{subfigure}
\usepackage{graphicx}

\usepackage{amsmath,amsthm}
\usepackage{amssymb,wasysym}
\usepackage{mathrsfs}
\usepackage{cite}
\usepackage{color}
\usepackage{url}

\usepackage{bbm}

\usepackage{wrapfig}
\usepackage{verbatim}
\usepackage{dsfont}
\usepackage{amsmath}
\usepackage{algorithm}
 \usepackage{algorithmic}
\usepackage{xurl}

\newtheorem{theorem}{\textbf{Theorem}}
\newtheorem{lemma}{\textbf{Lemma}}
\newtheorem{corollary}{\textbf{Corollary}}

\newtheorem{assumption}{\textbf{Assumption}}
\newtheorem{proposition}{\textbf{Proposition}}
\renewcommand{\algorithmicrequire}{\textbf{Input:}}  
\renewcommand{\algorithmicensure}{\textbf{Output:}} 
\allowdisplaybreaks[3]
\usepackage{tikz}

\ifodd 1
\definecolor{darkgreen}{RGB}{0,200,0}
\else

\fi

\begin{document}
%

\title{HASFL: Heterogeneity-aware Split Federated Learning over Edge Computing Systems}

\author{Zheng Lin, Zhe Chen~\IEEEmembership{Member,~IEEE}, Xianhao Chen,~\IEEEmembership{Member,~IEEE}, Wei Ni,~\IEEEmembership{Fellow,~IEEE}, and Yue Gao,~\IEEEmembership{Fellow,~IEEE}  
\thanks{Z. Lin and X. Chen are with the Department of Electrical and Electronic Engineering, University of Hong Kong, Pok Fu Lam, Hong Kong, China (e-mail: linzheng@eee.hku.hk; xchen@eee.hku.hk).}
\thanks{Z. Chen and Y. Gao are with the Institute of Space Internet, Fudan University, Shanghai 200438, China, and the School of Computer Science, Fudan University, Shanghai 200438, China (e-mail: zhechen@fudan.edu.cn;
 gao.yue@fudan.edu.cn).}
 \thanks{W. Ni is with Data61, CSIRO, Marsfield, NSW 2122, Australia, and the School of Computing Science and Engineering, and the University of New South Wales, Kensington, NSW 2052, Australia (e-mail:
wei.ni@ieee.org).}
\thanks{\textit{(Corresponding author: Xianhao Chen)}}
}

%
%


\markboth{Journal of \LaTeX\ Class Files,~Vol.~14, No.~8, August~2015}%
{Shell \MakeLowercase{\textit{et al.}}: Bare Advanced Demo of IEEEtran.cls for IEEE Computer Society Journals}


\IEEEtitleabstractindextext{
\begin{abstract}
Split federated learning (SFL) has emerged as a promising paradigm to democratize machine learning (ML) on edge devices by enabling layer-wise model partitioning. However, existing SFL approaches suffer significantly from the straggler effect due to the heterogeneous capabilities of edge devices. To address the fundamental challenge, we propose adaptively controlling batch sizes (BSs) and model splitting (MS) for edge devices to overcome resource heterogeneity. We first derive a tight convergence bound of SFL that quantifies the impact of varied BSs and MS on learning performance. Based on the convergence bound, we propose HASFL, a heterogeneity-aware SFL framework capable of adaptively controlling BS and MS to balance communication-computing latency and training convergence in heterogeneous edge networks. 
Extensive experiments with various datasets validate the effectiveness of HASFL and demonstrate its superiority over state-of-the-art benchmarks.

\end{abstract}

}

\maketitle

\IEEEdisplaynontitleabstractindextext

%
\IEEEpeerreviewmaketitle

\section{Introduction}\label{Intro}


Conventional machine learning (ML) frameworks predominantly rely on centralized learning (CL), where raw data is gathered and processed at a central server for model training. However, CL is 
often impractical due to its high communication latency, increased backbone traffic, and privacy risks~\cite{liu2023optimizing,lin2025hsplitlora,deng2022actions,chen2022federated}. To address these limitations, federated learning (FL)~\cite{mcmahan2017communication,konevcny2016federated} has emerged as a promising alternative that allows participating devices to collaboratively train a shared model via exchanging model parameters (e.g., gradients) rather than raw data, thereby protecting data privacy and reducing communication costs~\cite{lin2024fedsn,hu2024accelerating}. Despite its advantage, on-device training of FL poses a significant challenge for its deployment on resource-constrained edge devices as ML models scale up~\cite{fang2024automated,lin2025leo}. Considering the fact that training costs significantly more latency and memory space than model inference, training large-sized models, such as Gemini Nano-2 with 3.25 billion parameters (3GB for 32-bit floats), is computationally prohibitive for resource-constrained edge devices~\cite{team2023gemini}.

Split learning (SL)~\cite{vepakomma2018split} has emerged as a viable solution to overcoming the weaknesses of FL, which offloads the major workload from edge devices to a more powerful central server via layer-wise model splitting, alleviating computing workload and memory requirements on edge devices~\cite{wei2025pipelining,lin2023pushing,lyu2023optimal}. 
The original SL framework, named vanilla SL~\cite{vepakomma2018split}, employs an inter-device sequential training where the central server collaborates with one edge device after another for co-training, however, at the expense of excessive training latency. To address this, a prominent variant of SL, split federated learning (SFL)~\cite{thapa2022splitfed}, integrates the strengths of FL and SL to enable parallel training across edge devices. As shown in Fig.~\ref{HASFL}, SFL not only involves model splitting but also adheres to FL principles, requiring the fed server to periodically aggregate client sub-models from edge devices in a synchronous manner, akin to FedAvg~\cite{mcmahan2017communication}.


\begin{figure}[t!]
\centering
\includegraphics[width=8.7cm]{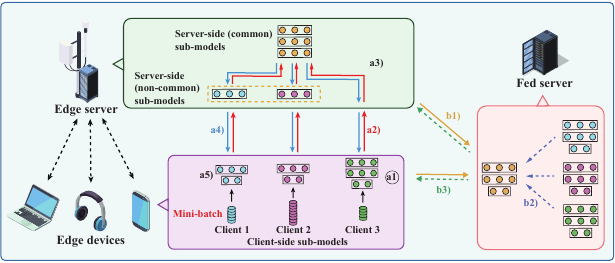}
\caption{The illustration of HASFL over edge computing systems, where a1) and a5) denote client-side forward propagation (FP) and backward pass (BP), a3) represents server-side model FP and BP, a2) and a4) are activations and activations' gradients transmissions, b1) and b3) denote sub-model uploading and downloading, and b2) represents client-side model aggregation.
}
\label{HASFL}
\end{figure}

Unfortunately, deploying SFL in heterogeneous edge systems poses significant challenges. The SFL framework enforces a synchronous aggregation protocol that necessitates every device to complete its local training before model aggregation. However, the heterogeneous computing and communication resources across edge devices lead to significantly varied training delays, causing a severe straggler effect\footnote{The straggler effect is a primary bottleneck in distributed learning, referring to the problem of edge devices with slower training speeds (i.e., stragglers) impeding the overall model training progress. Specifically, all edge devices must wait for the slowest one to complete model training before proceeding.}~\cite{chen2021distributed,xiao2023time,lee2017speeding}. Although existing studies~\cite{wu2023split,lin2024adaptsfl,lee2024game,lin2024hierarchical} attempt to mitigate this issue by determining varied cut layers for different devices, optimizing model splitting (MS) alone is insufficient as MS \textit{lacks the flexibility} of adjusting client-side workload. For instance, while splitting a convolutional neural network (CNN) at a shallow point leads to lower computing workload on clients, it also increases communication overhead due to a larger output size in early layers~\cite{lin2024adaptsfl}.

\begin{figure}[t]
\setlength\abovecaptionskip{3pt}
\centering
\subfigure[Test accuracy versus epochs.]{
    \includegraphics[width=0.449\linewidth]{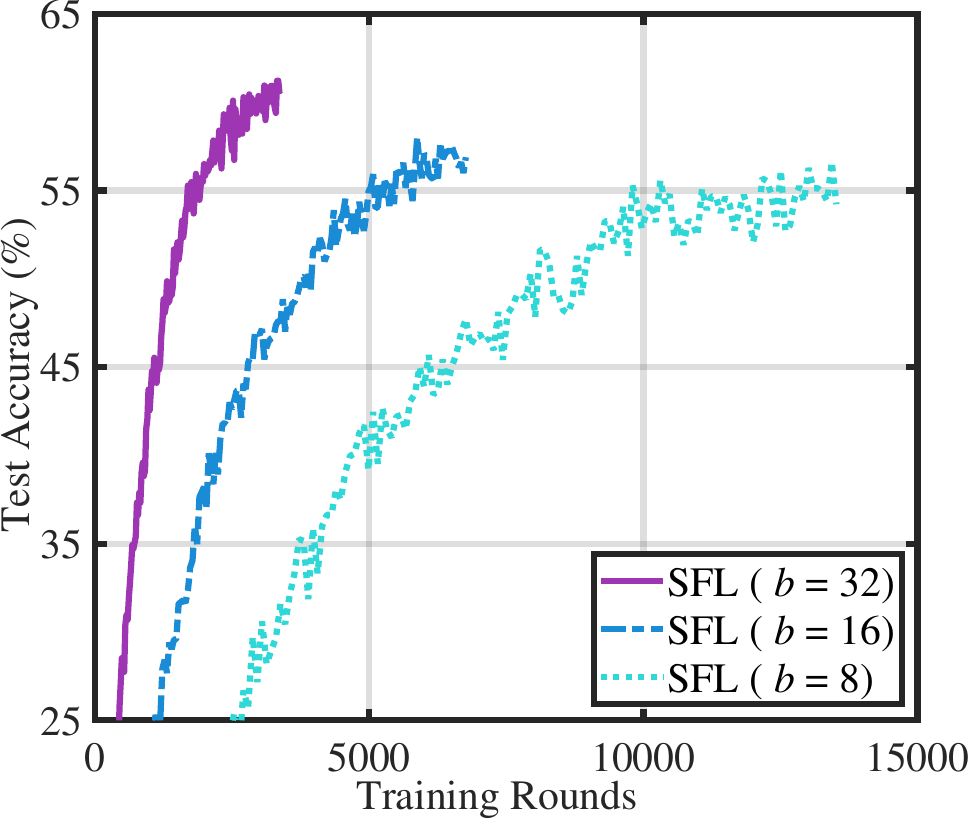}
    \label{sfig:motivation_1_different_batchsize_round}
}
\subfigure[Per-round training latency.]{
    \includegraphics[width=0.422\linewidth]{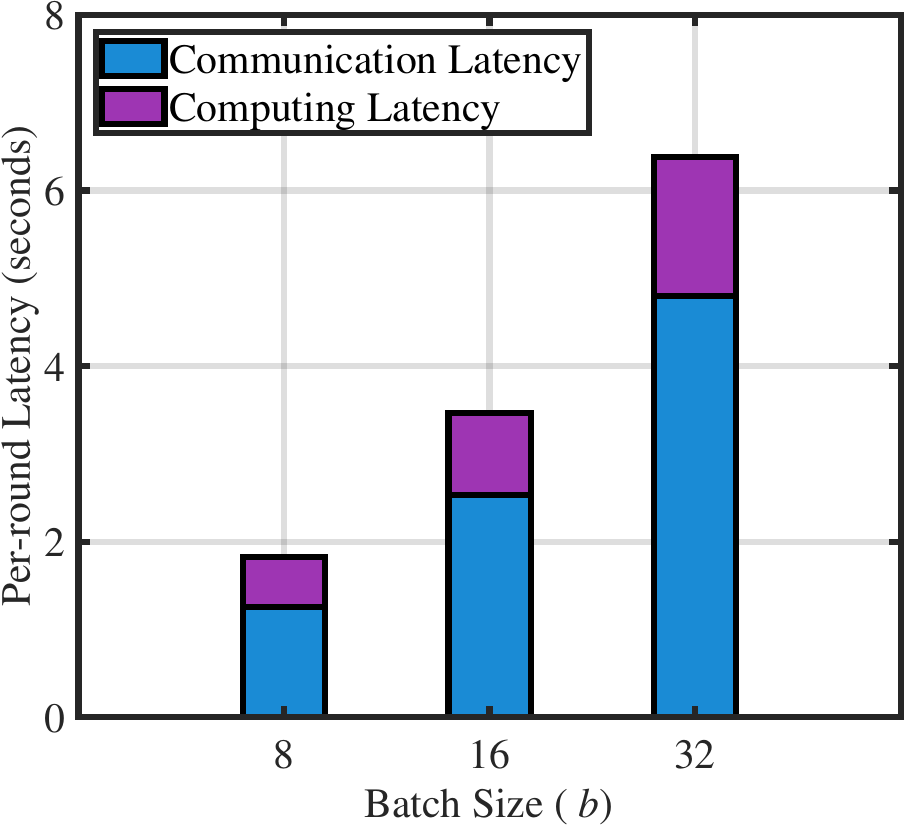}
    \label{sfig:moti_1_batchsize_per_round_latency}
}
    \caption{ The impact of BS on training performance and per-round training latency. 
    Fig.~\ref{sfig:motivation_1_different_batchsize_round} shows the performance for test accuracy versus number of epochs, and Fig.~\ref{sfig:moti_1_batchsize_per_round_latency} illustrates the effect of BS on per-round training latency. The experiment is conducted on the CIFAR-10 dataset under the non-IID setting with $L_c=8$ and $I=15$.}
    \label{fig:motivation_1}
    \vspace{-0.8em}
\end{figure}

\begin{figure}[t]
\setlength\abovecaptionskip{3pt}
\centering
\subfigure[Test accuracy versus epochs.]{
    \includegraphics[width=0.44300\linewidth]{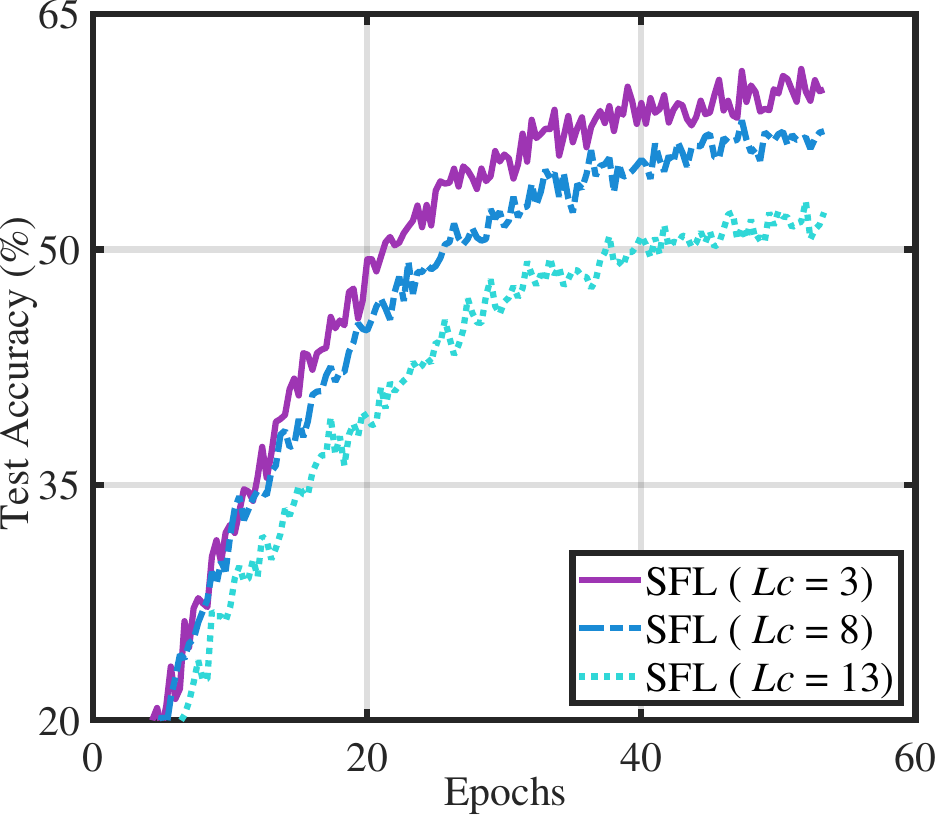}
    \label{sfig:motivation_2_different_cut}
}
\hspace{-1ex}
\subfigure[Computing and communication overhead.]{
    \includegraphics[width=0.494\linewidth]{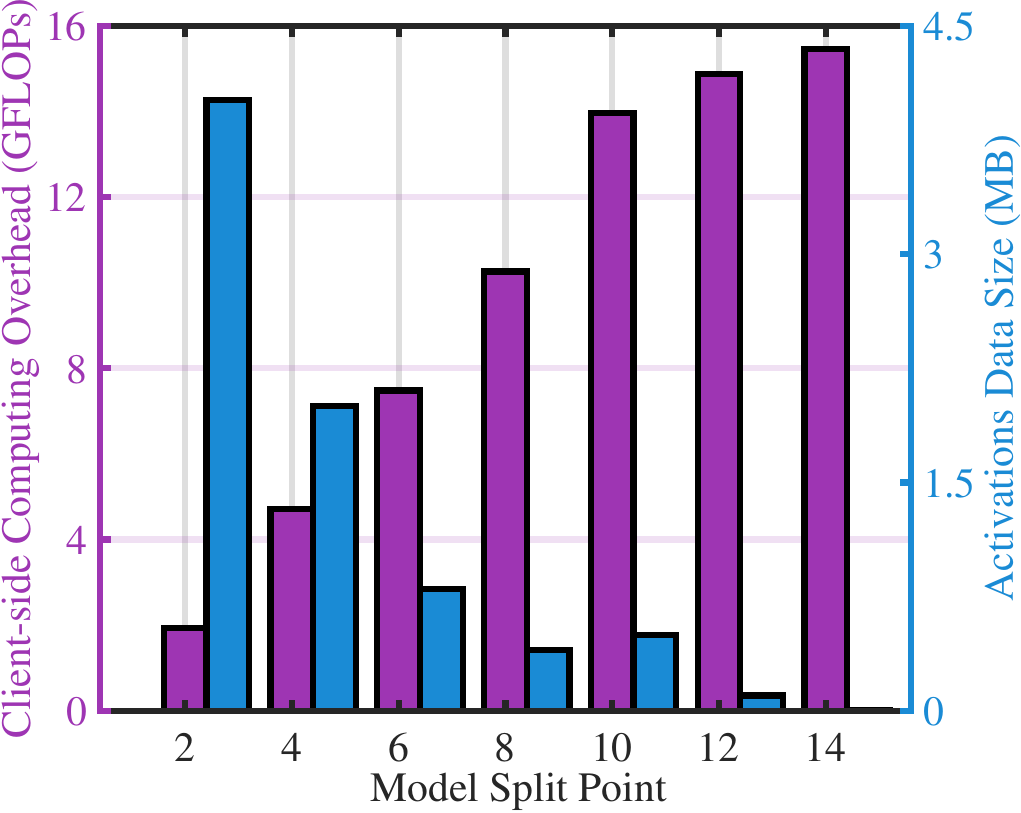}
    \label{sfig:motivation_2_cut_comput_commu}
}
    \caption{ The impact of MS on computing and communication overhead as well as training performance. 
    Fig.~\ref{sfig:motivation_2_cut_comput_commu} shows the computing and communication overhead of SFL at different model split points.  Fig.~\ref{sfig:motivation_2_different_cut} presents the performance for test accuracy versus the number of epochs. The experiment is conducted on the CIFAR-10 dataset under the non-IID setting with $b=16$ and $I=15$. }
    \label{fig:motivation_2}
\end{figure}


To effectively alleviate the straggler effect, BS serves as another key control variable. Our key insight is that \textit{both} the communication and computing workload of an edge device scales approximately linearly with its BS, making BS a flexible lever for customizing client-side workload. Moreover, BS and MS significantly influence the performance of SFL not only in latency but also in training accuracy: i) \textbf{Impact of BS:} As shown in Fig.~\ref{fig:motivation_1}, increasing parameter $b$ (batch size) expedites model convergence, but this improvement comes at the cost of increased per-round communication-computing latency. In particular, smashed data (activations) communication latency is linearly proportional to BSs of edge devices -- a phenomenon that does not exist in conventional FL settings. ii) \textbf{Impact of MS:} As illustrated in Fig.~\ref{fig:motivation_2}, MS introduces a complicated trade-off among computing, communications, and training accuracy. Selecting a shallower cut layer (i.e., smaller $L_c$) enhances training accuracy by allocating a larger portion of the model as the server-side sub-model\footnote{In SFL, the aggregation frequencies of client-side and server-side sub-models may differ. Specifically, server-side sub-models co-located on edge servers can be synchronized in each training round without introducing communication overhead. In contrast, client-side sub-model aggregation incurs communication costs and is therefore performed periodically over several training rounds. A shallower cut layer indicates a larger portion of the model being aggregated at each round (i.e., more frequent model aggregation), leading to better training convergence~\cite{wang2019adaptive}.}. However, it also incurs higher communication overhead due to the larger output dimensions of shallower layers in CNNs.



Optimizing BS and MS necessitates a fundamental trade-off analysis of SFL under resource constraints, which is not yet studied. To fill the void, in this paper, we propose HASFL, a \underline{h}eterogeneity-\underline{a}ware \underline{SFL} framework with adaptive BS and MS to jointly determine varied BS and MS under heterogeneous SFL systems. We first derive an analytical convergence upper bound that quantifies the impact of BS and MS on training convergence, and then formulate the problem to minimize total training latency by jointly optimizing BS and MS.  The challenges addressed by HASFL are twofold: i) From the convergence analysis perspective, unlike FL with varied batch sizes~\cite{liu2020adaptive,ma2021adaptive,liu2023dynamite}, the server-side sub-model training in SFL is equivalent to concatenating the entire batch from all clients, whereas the client-side sub-model training runs on its local batch. The discrepancy also makes BS and MS jointly affect training convergence, thus resulting in intricate convergence behavior. We highlight that this paper provides the first convergence bound of SFL with \textit{varied} BSs and cut layers for clients. ii) From a system optimization perspective, both BS and MS balance the tradeoff between training convergence and per-round latency, resulting in a complicated convergence-latency tradeoff.

The main contributions of this paper are summarized as follows.

\begin{itemize}
\item We propose HASFL, a heterogeneity-aware SFL framework, which controls BS and MS to accelerate SFL in heterogeneous resource-constrained edge computing systems.
\item We establish the \textit{first} convergence bound that rigorously quantifies the impact of varied BSs and MS on SFL, and formulate a new problem to jointly optimize BS and MS.  
\item We decompose the new problem into two tractable sub-problems of optimizing BS and MS, and develop an efficient optimization method to solve the sub-problems alternately. 
\item  We conduct extensive simulations across various datasets to validate our analysis and demonstrate the superiority of the proposed solution over the state-of-the-art in training accuracy and convergence speed.
\end{itemize}

The remainder of this paper is organized as follows. Section~\ref{Rel_Work} elaborates on related work and Section~\ref{sec_sfl} introduces the system model and HASFL framework. Section~\ref{convergence_HASFL} provides the convergence analysis of HASFL. We formulate the optimization problem in Section~\ref{prob_formu} and offer the corresponding solution approach in Section~\ref{solu_appro}. Section~\ref{simu_results} provides the simulation results. Finally, concluding remarks are presented in Section~\ref{conclu}.

\section{Related Work}\label{Rel_Work}

Some studies have been conducted to enhance the training performance of FL via BS control~\cite{shi2022talk,ma2021adaptive,liu2020adaptive,liu2023dynamite,2018dont,yu2019computation}. Shi~\textit{et al.}~\cite{shi2022talk} proposed a dynamic BS assisted FL framework that dynamically controls BS to minimize the total latency of FL over mobile devices. Ma~\textit{et al.}~\cite{ma2021adaptive} developed an efficient FL algorithm that adaptively adjusts BS with scaled learning rate for heterogeneous devices to reduce their waiting time and improve the model convergence rate. Liu~\textit{et al.}~\cite{liu2020adaptive} jointly optimized the BS, compression ratio, and
spectrum allocation to maximize the convergence rate under the given training latency constraint. Liu~\textit{et al.}~\cite{liu2023dynamite} first established the convergence bound for training error considering heterogeneous datasets across devices and then derived closed-form solutions for co-optimized BS and aggregation frequency configuration. Smith \textit{et al.}~\cite{2018dont} and Yu \textit{et al.}~\cite{yu2019computation} empirically designed dynamic BS schemes to establish a more communication-efficient FL framework. Nevertheless, these existing BS control schemes cannot be directly applied to SFL, as aggregation discrepancies between client-side and server-side sub-models lead to distinct convergence behaviors from conventional FL. This phenomenon, as alluded to earlier, causes the intricate relationship between BS and MS in SFL.

MS heavily impacts the training latency of SL, as it determines the computing load distribution between devices and server, the size of smashed data, and model convergence. Tremendous research efforts have been made to determine the optimal model split point for SL~\cite{wu2023split,lee2024game,lin2023efficient}. Wu \textit{et al.}~\cite{wu2023split} developed a cluster-based parallel SL framework and a joint MS, device clustering, and subchannel allocation optimization algorithm to minimize the training latency under heterogeneous device capabilities and dynamic network conditions. Lee \textit{et al.}~\cite{lee2024game} proposed a hierarchical Stackelberg game framework for SFL that jointly optimizes MS and incentive mechanisms to to balance the training load between the server and clients, ensuring the necessary level of privacy for the clients. Lin \textit{et al.}~\cite{lin2023efficient} designed an efficient parallel SL framework that aggregates last-layer gradients to reduce the dimensionality of activations’ gradients and devised a joint MS, subchannel allocation, and power control to reduce the overall training latency. Unfortunately, these works fail to consider the impact of MS on training convergence.

The prior work~\cite{han2024convergence} conducted a convergence analysis for SFL, but does not consider the impact of MS, which plays a critical role in the system performance of SFL. To the best of our knowledge, our recent works~\cite{lin2024adaptsfl,lin2024hierarchical} were the first to provide a theoretical convergence analysis for SFL by considering the impact of MS. Since these studies assumed that all edge devices employ the same BS, neither convergence analysis nor system optimization of SFL can be directly applied to HASFL, both of which are non-trivial.

\section{The Heterogeneity-aware SFL Framework}\label{sec_sfl}
This section presents the system model in Section~\ref{Sys_Model} and details the proposed heterogeneity-aware SFL framework, HASFL, in Section~\ref{subsec_adapt}.

\subsection{System Model}\label{Sys_Model}

As illustrated in Fig.~\ref{HASFL}, we consider a typical edge computing scenario of HASFL, which comprises three components:

\begin{itemize}
\item \textbf{Edge devices:} We consider that each client possesses an edge device capable of executing local computations, i.e., the client-side forward propagation (FP) and backward pass (BP). The set of participating edge devices is represented as $\mathcal{N} = \left\{ {1,2,...,N} \right\}$, where $N$ is the total number of edge devices. For the $i$-th edge device , its local dataset is  ${\mathcal{D}_i} = \left\{ {{{\bf{x}}_{i,k}},{y_{i,k}}} \right\}_{k = 1}^{{|\mathcal{D}_i|}}$, where ${{\bf{x}}_{i,k}}$  and ${{{y}}_{i,k}} $ are the $k$-th input data sample and its corresponding label, respectively. The client-side sub-model of the $i$-th edge device  is represented as ${{\bf{w}}_{c,i}}$.

\item \textbf{Edge server:} The edge server is a computing entity with powerful computing capability, responsible for executing the server-side model training. The server-side sub-model of the $i$-th edge device is denoted by ${{\bf{w}}_{s,i}} = \left[{{\bf{h}}_{s}};{{\bf{h}}_{m,i}}\right]$, where ${\bf{h}}_{s}$ and ${\bf{h}}_{m,i}$ are the server-side common and server-side non-common sub-models, respectively. The common sub-model is the shared component across all clients and is synchronized in every training round. The non-common sub-model arises from the discrepancies in the number of layers maintained on the server across clients, and it is periodically aggregated with client-side sub-model every several training rounds. Moreover, the edge server is also tasked with network information collection, such as device computing capabilities and channel conditions, to implement optimization decisions of HASFL.

\item \textbf{Fed server:} The fed server takes charge of the synchronization of client-side sub-models and server-side non-common sub-models, periodically aggregating the client-side sub-models together with server-side non-common sub-models across all participating edge devices. For privacy concerns, fed and edge servers usually belong to separate parties, as possessing both the client-side sub-models and smashed data could potentially enable a malicious server to reconstruct the original user data~\cite{pasquini2021unleashing}. 
\end{itemize}

The global model is represented as ${\bf{w}} = \left[ {{{\bf{w}}_{s,i}};{{\bf{w}}_{c,i}}} \right] $. The goal of SL is to obtain the optimal global model ${{\bf{w}}^{\bf{*}}}$ by
minimizing the following finite-sum non-convex global loss function:
\begin{align}\label{minimiaze_loss_function}
\mathop {\min }\limits_{\bf{w}} f\left( {\bf{w}} \right) \buildrel \Delta \over = \mathop {\min }\limits_{\bf{w}} {\frac{{{1}}}{N}} \sum\limits_{i = 1}^N {f_i}({\bf{w}}),
\end{align}
where ${f_i}\left( {\bf{w}} \right) \buildrel \Delta \over = {\mathbb{E}_{{\xi _i} \sim {\mathcal{D}_i}}}[{F_i}\left( {{\bf{w}};{\xi _i}} \right)]$ denotes the local loss function of the $i$-th edge device  and $\xi _i$ is training randomness\footnote{Training randomness refers to the stochasticity introduced during model training, primarily due to mini-batch data sampling and random data shuffling within the local dataset~\cite{Karimired2018,yu2019linear,karimireddy2020scaffold}.} from the local dataset $\mathcal{D}_i$. For conciseness, we consider that all edge devices have local datasets with uniform size. Consistent with the standard setting~\cite{Karimired2018,yu2019linear,karimireddy2020scaffold}, it is assumed that the stochastic gradient is an unbiased estimate of the true gradient, i.e., $\mathbb{E}_{\xi_{i}^{t}\sim \mathcal{D}_{i}}[\nabla F_{i}(\mathbf{w}^{t-1}_{i};\xi_{i}^{t}) \vert \boldsymbol{\xi}^{[t-1]}] = \nabla f_{i}(\mathbf{w}^{t-1}_{i})$, where ${\nabla _{\bf{w}}}F({\bf{w}}; \xi)$ is the gradient of the function $F({\bf{w}}; \xi)$ with respect to model parameter ${\bf{w}}$ and $\boldsymbol{\xi}^{[t-1]} \overset{\Delta}{=} [\xi_{i}^{\tau}]_{i\in\{1,2,\ldots,N\}, \tau\in\{1,\ldots,t-1\}}$ represents all training randomness up to the $(t-1)$-th training round.

To solve problem~\eqref{minimiaze_loss_function}, conventional SFL frameworks employ a uniform BS and fixed MS across edge devices throughout model training. However, this scheme leads to severe straggler effects in heterogeneous edge computing systems, significantly slowing down model convergence. Motivated by this, we propose a HASFL featuring heterogeneity-aware BS and MS, as will be detailed in the following section.

\subsection{HASFL Procedure}\label{subsec_adapt}
This section presents the workflow of the proposed HASFL framework. The salient feature of HASFL lies in heterogeneity-aware BS and MS. By jointly optimizing BS and MS, HASFL can significantly reduce the overall training latency while retaining the desired learning performance.

Before model training starts, the edge server initializes the ML model. Then, the optimal BS is determined (see Section~\ref{solu_appro}), and the global model is partitioned into client-side and server-side sub-models via optimal MS (see Section~\ref{solu_appro}). Afterwards, HASFL aggregates client-side sub-models and updates BS and MS based on device computing and communication resources every $I$ training rounds. This process loops until the model converges. The training procedure of HASFL consists of two primary stages: split training and client-side model aggregation. The split training is executed in every training round, while the client-side model aggregation is triggered every $I$ training rounds. As illustrated in Fig.~\ref{HASFL}, for any training round $t$, i.e., $t \in \mathcal{R} = \left\{ {1,2,...,R} \right\}$, the training workflow of HASFL is presented as follows.

\textit{a. The split training stage:} This stage involves client-side and server-side model updates and smashed data exchange in each training round, comprising the following five steps.

\textit{a1) Client-side model forward propagation:} This step involves the parallel execution of client-side FP of edge devices. For the arbitrary $t$-th training round, each edge device $i$ randomly samples a mini-batch ${\mathcal{B}^t_i} \subseteq {\mathcal{D}_i}$ containing $b_i$ data samples from its local dataset for model training. For the $t$-th training round, the input data and corresponding labels of the mini-batch in training round $t$ are represented as ${{\bf{x}}^t_i}$ and ${{\bf{y}}^t_i}$, respectively. The client-side sub-model of the $i$-th edge device trained up to the $(t-1)$-th round, is denoted by ${\bf{w}}^{t-1}_{c,i}$. After feeding a mini-batch into the client-side sub-model, activations are generated at the cut layer. The activations of the $i$-th edge device  are expressed as
\begin{align}\label{stage_1_1}
{{\bf{a}}^t_i} = \varphi \left( {{\bf{x}}^t_i};{{\bf{w}}^{t-1}_{c,i}} \right), 
\end{align}
where $\varphi\left( {{\bf{x}};{\bf{w}}} \right)$ is the mapping function between input data ${\bf{x}}$ and its predicted value given model parameter ${\bf{w}}$.

\textit{a2) Activations transmissions:} 
After the client-side FP is completed, each edge device transmits its generated activations and corresponding labels to the edge server (usually over wireless channels).

\textit{a3) Server-side model forward propagation and backward pass:} The edge server gathers activations from participating edge devices and then feeds these activations into the server-side sub-models to perform server-side FP. For the $i$-th edge device, the predicted value of the server-side sub-model is represented as
\begin{align}\label{stage_1_3}
{\bf{\hat y}}^t_i = \varphi\left( {{\bf{a}}^t_i;{{\bf{w}}^{t-1}_{s,i}}} \right), 
\end{align}
where ${{\bf{w}}^{t-1}_{s,i}} = \left[{{\bf{ h}}^{t-1}_{s}};{{\bf{ h}}^{t-1}_{m,i}}\right]$; ${\bf{ h}}^{t-1}_{s}$ and ${\bf{h}}^{t-1}_{m,i}$ are the server-side common and server-side non-common sub-models, respectively. The predicted value and labels are utilized to compute the loss function and derive the server-side sub-model's gradients.

Due to heterogeneous cut layers across edge devices, the server-side sub-model consists of two parts: the common sub-model, shared by all edge devices, and the non-common sub-model, formed by the extra layers from edge devices with deeper cut layers. These structural disparities necessitate distinct update mechanisms for each component.

For the server-side common sub-model, the edge server updates it\footnote{{The edge server can update the model for multiple edge devices in either a serial or parallel manner, which does not impact training performance as aggregation occurs every round. Here, we formulate the parallel manner.}} every training round by averaging the common sub-model updates from all edge devices, i.e.,
\begin{align}\label{stage_5_2}
{\bf{h}}_s^t = \frac{1}{N}\sum\limits_{i = 1}^N {{\bf{h}}_{s,i}^t},
\end{align}
where ${\bf{h}}_{s,i}^t \leftarrow {\bf{h}}_{s,i}^{t - 1} - \gamma {\nabla _{{{\bf{h}}_s}}}{F_i}({\bf{h}}_{s,i}^{t - 1};\xi _i^t)$ is the server-side common sub-model of the $i$-th edge device, ${\nabla _{{{\bf{h}}_s}}}{F_i}({\bf{h}}_{s}^{t - 1};\xi _i^t)$ is the stochastic gradients of server-side common sub-models of the $i$-th edge device, and $\gamma$ represents the learning rate. Since the aggregation step in Eqn.~\eqref{stage_5_2} does not incur any communication overhead, it can be executed at every training round to expedite training convergence~\cite{lin2024adaptsfl,lin2024hierarchical}. Consequently, the update process for the server-side common sub-model is equivalent to centralized training with stochastic gradient descent.

In contrast, the server-side non-common sub-models, specific to the edge devices with deeper cut layers, are updated individually on their respective devices without cross-device aggregation. The update procedure for the server-side non-common sub-model of the $i$-th edge device is given by
\begin{align}\label{non_commen_update}
{\bf{h}}_{m,i}^t \leftarrow {\bf{h}}_{m,i}^{t - 1} - \gamma {\nabla _{{{\bf{h}}_m}}}{F_i}({\bf{h}}_{m,i}^{t - 1};\xi _i^t), 
\end{align}
where ${\nabla _{{{\bf{h}}_m}}}{F_i}({\bf{h}}_{m,i}^{t - 1};\xi _i^t)$ gives the stochastic gradients of server-side non-common sub-models of the $i$-th edge device.

\textit{a4) Activations' gradients transmissions:} After the completion of server-side BP, the edge server sends the activations' gradients to the respective edge devices.

\textit{a5) Client-side model backward pass:} Each edge device updates its client-side sub-model based on the received activations' gradients. For the $i$-th edge device, the client-side sub-model is updated through
\begin{align}\label{client_side_update}
{{\bf{w}}^t_{c,i}} \leftarrow {{\bf{w}}^{t-1}_{c,i}} - {\gamma } {\nabla _{{{\bf{w}}_c}}}{F_i}({\bf{w}}_{c,i}^{t - 1};\xi _i^t),
\end{align}
where ${\nabla _{{{\bf{w}}_c}}}{F_i}({\bf{w}}_{c,i}^{t - 1};\xi _i^t)$ represent the stochastic gradients of client-side sub-models of the $i$-th edge device.

\textit{b. The client-side model aggregation stage:} The client-side model aggregation stage aims to aggregate forged client-specific sub-models (including client-side sub-models and server-side non-common sub-models specific to each client) on the fed server. This stage is executed every $I$ training rounds and comprises the following three steps.

\textit{b1) Sub-model uploading:} In this step, edge devices and the edge server simultaneously upload the client-side sub-models and corresponding server-side non-common sub-models to the fed server, serving as the basis for the subsequent forged client-specific sub-model aggregation.

\textit{b2) Client-side model aggregation:} The fed server pairs and assembles the received client-side sub-models and server-side non-common sub-models to forge the client-specific models. Then, these forged client-specific models are aggregated across all participating devices, as given by 
\begin{align}\label{h_c_define}
{\bf{h}}_c^t = \frac{{{1}}}{N} \sum\limits_{i = 1}^N {{\bf{h}}_{c,i}^t} ,
\end{align}
where ${{\bf{h}}^t_{c,i}} = \left[{{\bf{h}}^t_{m,i}};{{\bf{ w}}^t_{c,i}}\right]$ is the forged client-specific model of the $i$-th edge device.

\textit{b3) Sub-model downloading:} After the client-side model aggregation is completed, the fed server transmits updated client-side sub-models and server-side non-common sub-models back to corresponding edge devices and the edge server, respectively.

The workflow of the HASFL training framework is presented in~\textbf{Algorithm~\ref{HASFL_procedure}}, where $E$ denotes the number of stochastic gradients involved throughout model training.

\begin{algorithm}[t]
     \small
	\renewcommand{\algorithmicrequire}{\textbf{Input:}}
	\renewcommand{\algorithmicensure}{\textbf{Output:}}
	\caption{The HASFL Training Framework}\label{HASFL_procedure}
	\begin{algorithmic}[1]
 \REQUIRE   $I$, $\gamma$, $E$, ${\cal N}$ and $\cal{D}$.
		\ENSURE ${{\bf{w}}^{\bf{*}}}$. 
		\STATE Initialization: ${{\bf{w}}^{{0}}}$, ${{b}^{{0}}}$, ${{\bf{w}}_i^{{0}}} \leftarrow {{\bf{w}}^{{0}}}$, $b_i^0 \leftarrow {{b}^{{0}}}$, $\tau \leftarrow 0$, and $\rho  \leftarrow 0$.
          \WHILE{$\sum\limits_{\eta  = 0}^\tau  {\sum\limits_{i = 1}^N {b_i^\eta } }  \le E$}

          \STATE
          \STATE /** {Runs} {on} {edge} {devices} **/
          \FORALL {edge device ${i \in \,{\cal N}}$ in parallel}
            \STATE  ${{\bf{a}}^\tau_i} \leftarrow \varphi \left( {{\bf{x}}^\tau_i};{{\bf{w}}^{\tau-1}_{c,i}} \right)$
            \STATE Send $\left( {{{\bf{a}}^\tau_i},{\bf{y}}_i^\tau} \right)$ to the edge server
          \ENDFOR

	   \STATE
          \STATE /** {Runs} {on} {edge} {server} **/
          \STATE${\bf{\hat y}}^\tau_i = \varphi\left( {{\bf{a}}^\tau_i;{{\bf{w}}^{\tau-1}_{s,i}}} \right)$
          \STATE Calculate loss function value $f\left( {{\bf{w}}^{\tau - 1}} \right)$
          \STATE ${\bf{h}}_s^\tau \leftarrow {\bf{h}}_s^{\tau - 1} - \gamma \frac{1}{N}\sum\limits_{i = 1}^N {{\nabla _{{{\bf{h}}_s}}}{F_i}({\bf{h}}_s^{\tau - 1};\xi _i^\tau)} ,$ 
          \STATE ${\bf{h}}_{m,i}^\tau \leftarrow {\bf{h}}_{m,i}^{\tau - 1} - \gamma {\nabla _{{{\bf{h}}_m}}}{F_i}({\bf{h}}_{m,i}^{\tau - 1};\xi _i^\tau)$
          \STATE Send activations' gradients  to corresponding edge devices

	   \STATE
          \STATE /** {Runs} {on} {edge} {devices} **/
         \FORALL {edge device ${i \in \,{\cal N}}$ in parallel}
           \STATE ${{\bf{w}}^t_{c,i}} \leftarrow {{\bf{w}}^{\tau-1}_{c,i}} - {\gamma } {\nabla _{{{\bf{w}}_c}}}{F_i}({\bf{w}}_{c,i}^{\tau - 1};\xi _i^\tau)$
        \ENDFOR
        \STATE
        \STATE /** {Runs} {on} the {fed} {server} **/
        \IF{$\tau$ mod $I=0$}
          \STATE Determine ${{\bf b}^\tau } = [b_1^\tau ,b_2^\tau ,...,b_N^\tau ]$ and ${\boldsymbol{\mu}}^\tau$ based on~\textbf{Algorithm~\ref{BCD-based}}
          \STATE Forge client-side sub-models ${{\bf{h}}^\tau_{c,i}} = \left[{{\bf{h}}^t_{m,i}};{{\bf{ w}}^\tau_{c,i}}\right]$
          \STATE ${\bf{h}}_c^\tau = \frac{{{1}}}{N} \sum\limits_{i = 1}^N {{\bf{h}}_{c,i}^\tau}$
          \STATE ${\bf{h}}_{c,i}^\tau \leftarrow {\bf{h}}_c^\tau$
          \ELSIF{$\tau$ mod $I\ne0$}
          \STATE Determine ${{\bf b}^\tau } = [b_1^\tau ,b_2^\tau ,...,b_N^\tau ]$ based on~\textbf{Theorem~\ref{theorem2}} 
        \ENDIF
        \STATE
        \STATE Update $\tau \leftarrow \tau+1$
        
          \ENDWHILE          
	\end{algorithmic}  
\end{algorithm}

\section{Convergence Analysis of HASFL}\label{convergence_HASFL}

This section presents the convergence analysis of SFL, characterizing the effect of BS and MS on training convergence. This analysis serves as the theoretical foundation for designing an efficient iterative optimization algorithm in Section~\ref{solu_appro}.

We define the global model of the $i$-th edge device at the $t$-th training round as ${\bf{w}}_i^t = \left[ {{\bf{h}}_{s,i}^t;{\bf{h}}_{c,i}^t} \right]$, which comprises server-side common sub-model ${\bf{h}}_{s,i}^t$ and forged client-specific sub-model ${\bf{h}}_{c,i}^t$. The server-side common sub-model is updated and synchronized across all devices every training round, whereas the forged client-specific sub-models are aggregated every $I$ training rounds.

Due to non-uniform model partitioning, different edge devices retain varying numbers of layers on the client side. Since the client-side and server-side sub-models follow different aggregation schedules, with only the client-side sub-models being aggregated periodically every $I$ rounds, we define the model split point $L_c$ as the maximum depth (i.e., number of layers) of client-specific sub-models among all participating edge devices. The gradients of forged client-specific and server-side common sub-models are given by
\begin{equation}\label{g_ci}
{\bf{g}}_{c,i}^t = \left[ {{\nabla _{{{\bf{h}}_m}}}{F_i}({\bf{h}}_{m,i}^{t - 1};\xi _i^t);{\nabla _{{{\bf{w}}_c}}}{F_i}({\bf{w}}_{c,i}^{t - 1};\xi _i^t)} \right]
\end{equation}
and 
\begin{equation}
{\bf{g}}_{s,i}^t = \sum\limits_{i = 1}^N {{\nabla _{{{\bf{h}}_s}}}{F_i}({\bf{h}}_{s}^{t - 1};\xi _i^t)},
\end{equation}
where ${\bf{h}}_s^t = \frac{{{1}}}{N} \sum\limits_{i = 1}^N {{\bf{h}}_{s,i}^t}$. Therefore, the gradient of the global model of the $i$-th edge device is ${\bf{g}}_i^t = \left[ {{\bf{g}}_{s,i}^t; {\bf{g}}_{c,i}^t} \right]$.

In line with the seminal studies in distributed stochastic optimization~\cite{zhang2012communication,lian2017can,mania2017perturbed,lin2018don} that analyze convergence bound on the aggregated version of individual solutions, we conduct the convergence analysis of ${{\bf{w}}^t} = \frac{{{1}}}{N} \sum\limits_{i = 1}^N {\bf{w}}_i^t$. To analyze the convergence upper bound of HASFL, we consider the following two standard assumptions on loss functions~\cite{zhang2012communication,lian2017can,mania2017perturbed,lin2018don}:

\begin{assumption}[\textit{Smoothness of the local loss functions}]\label{asp:1}
\textit{Each local loss function ${f_i}\left( {\bf{w}} \right)$ is differentiable and $\beta $-smooth., i.e., for any $\mathbf{w}$ and ${{\bf{w'}}}$, we have}
    \begin{equation}
\left\| {\nabla_{\bf{w}} {f_i}\left( {\bf{w}} \right) - \nabla_{\bf{w}} {f_i}\left( {\bf{w'}} \right)} \right\| \le \beta \left\| {{\bf{w}} - {\bf{w'}}} \right\|,\;\forall i.
    \end{equation}
\end{assumption}

\begin{assumption}[\textit{Bounded variance and second-order moments of stochastic gradients}]\label{asp:2}
\textit{The stochastic gradients computed via mini-batch sampling have bounded variance and second-order moments for each layer: }
    \begin{equation}
\mathbb{E}_{\xi_{i}\sim \mathcal{D}_{i}} \Vert \nabla_{\bf{w}} F_{i}(\mathbf{w}; \xi_{i}) - \nabla_{\bf{w}} f_{i}(\mathbf{w})\Vert^{2} \leq  \sum\limits_{j = 1}^l  {{\frac{{\sigma _j^2}}{b}}}, \forall \mathbf{w}, \;\forall i, 
    \end{equation}
        \begin{equation}
\mathbb{E}_{\xi_{i}\sim \mathcal{D}_{i}} \Vert \nabla_{\bf{w}} F_{i}(\mathbf{w}; \xi_{i}) \Vert^{2} \leq \sum\limits_{j = 1}^l  {G _j^2}, \forall \mathbf{w}, \;\forall i ,
    \end{equation}
\textit{where $b$ is the mini-batch size, $l$ denotes the number of layers for model  $\mathbf{w}$, ${{\sigma _j^2}}$ is the constant, $\frac{{\sigma _j^2}}{b}$ and ${G_j^2}$ represent the bounded variance and second-order moments for the $j$-th layer of model $\bf w$, respectively. }
\end{assumption}

\begin{lemma} \label{lm:diff-avg-per-node}
Under {\bf{Assumption \ref{asp:1}}, \bf{Algorithm \ref{HASFL_procedure}}} ensures
\begin{align}
\mathbb{E} [\Vert {\mathbf{h}}_c^{t} - \mathbf{h}^{t}_{c,i}\Vert^{2}] \leq {\mathbbm{1}}_{\{I > 1\}} 4\gamma^{2} I^{2} \sum\limits_{j = 1}^{L_c}  {G _j^2}, \forall i, \forall t,
\end{align}
\textit{where $I$ denotes the client-side model aggregation interval, and ${\mathbf{h}}_c^{t}$ is defined in Eqn. \eqref{h_c_define} as the virtual aggregated version of client-side model at the $t$-th training round.
}
\end{lemma}

\begin{proof}
Consider an arbitrary training round $t\geq 1$. Let $t_{0} \leq  t$ be the most recent round at which client-side model aggregation took place, i.e., the largest $t_0$ satisfying $t_0 \bmod I = 0$ (Such $t_{0}$ must exist and $t - t_{0} \leq I$.). Recalling Eqn.~\eqref{non_commen_update}, Eqn.~\eqref{client_side_update} and Eqn.~\eqref{g_ci}, the forged
client-specific sub-model of the $i$-th edge device can be expressed as 
\begin{align}\label{lemma_eq_1}
\mathbf{h}_{c,i}^{t}=  \mathbf{h}_{c}^{t_0}  -\gamma \!\!\sum_{\tau = t_{0}+1}^{t}\mathbf{g}_{c,i}^{\tau}.
\end{align}
{By \eqref{h_c_define},} we have
\begin{align}\label{lemma_eq_2}
{\mathbf{h}}_c^{t}  = {\mathbf{h}}_c^{t_0} -\gamma\!\!\sum_{\tau = t_{0}+1}^{t} \frac{1}{N}\sum_{i=1}^{N}\mathbf{g}_{c,i}^{\tau}.
\end{align}

Subtracting Eqn.~\eqref{lemma_eq_2} from Eqn.~\eqref{lemma_eq_1} and taking the squared norm yields:
{\begin{align*}
&\mathbb{E} [\Vert \mathbf{h}^{t}_{c} - {\mathbf{h}}^{t}_{c,i} \Vert^{2}] \\=&  {\mathbbm{1}}_{\{I > 1\}} \mathbb{E} [\Vert \gamma\!\! \sum_{\tau = t_{0}+1}^{t} \frac{1}{N}\sum_{i=1}^{N}\mathbf{g}_{c,i}^{\tau} -\gamma\!\! \sum_{\tau = t_{0}+1}^{t}\!\!\!\!\mathbf{g}_{c,i}^{\tau} \Vert^{2}]\\
=& {\mathbbm{1}}_{\{I > 1\}} \gamma^{2} \mathbb{E}  [\Vert \sum_{\tau = t_{0}+1}^{t} \frac{1}{N}\sum_{i=1}^{N}\mathbf{g}_{c,i}^{\tau} -\!\!\sum_{\tau = t_{0}+1}^{t}\!\!\!\!\mathbf{g}_{c,i}^{\tau}\Vert^{2}]\\
\overset{(a)}{\leq}& {\mathbbm{1}}_{\{I > 1\}} 2\gamma^{2} \mathbb{E} [\Vert \sum_{\tau = t_{0}+1}^{t} \!\frac{1}{N}\sum_{i=1}^{N}\mathbf{g}_{c,i}^{\tau}\Vert^{2} + \Vert \!\!\!\sum_{\tau = t_{0}+1}^{t}\!\!\!\!\mathbf{g}_{c,i}^{\tau}\Vert^{2} ]\\
\overset{(b)}{\leq}& {\mathbbm{1}}_{\{I > 1\}} 2\gamma^{2} (t-t_{0}) \mathbb{E} [ \sum_{\tau = t_{0}+1}^{t}\!\Vert  \frac{1}{N}\sum_{i=1}^{N}\mathbf{g}_{c,i}^{\tau}\Vert^{2} + \!\!\!\!\sum_{\tau = t_{0}+1}^{t}\!\!\Vert \mathbf{g}_{c,i}^{\tau}\Vert^{2}]\\
\overset{(c)}{\leq}& {\mathbbm{1}}_{\{I > 1\}} 2\gamma^{2} (t-t_{0}) \mathbb{E} [\sum_{\tau = t_{0}+1}^{t} (\frac{1}{N}\sum_{i=1}^{N}\Vert \mathbf{g}_{c,i}^{\tau}\Vert^{2}) +\!\!\sum_{\tau = t_{0}+1}^{t}\!\!\!\Vert \mathbf{g}_{c,i}^{\tau}\Vert^{2}]\\
 \overset{(d)}{\leq} & {\mathbbm{1}}_{\{I > 1\}} 4\gamma^{2} I^{2} \sum\limits_{j = 1}^{L_c}  {G _j^2},
\end{align*}}%
where (a)-(c) follow by the inequality $\Vert \sum_{i=1}^{n} \mathbf{z}_{i}\Vert^{2} \leq n \sum_{i=1}^{n} \Vert \mathbf{z}_{i}\Vert^{2}$ for any vectors $\mathbf{z}_{i}$ and any positive integer $n$ (using $n=2$ in (a), $n=t-t_0$ in (b), and $n=N$ in (c)); and (d) follows from {\bf Assumption \ref{asp:2}}.
\end{proof}

\begin{theorem}\label{theorem1}
Under {\bf{Assumption \ref{asp:1}}} and {\bf{Assumption \ref{asp:2}}}, if $0 < \gamma \leq \frac{1}{\beta}$ in {\bf{Algorithm \ref{HASFL_procedure}}}, then for any total number of training rounds $R\geq 1$, we have 
{ \begin{equation}\label{convergence_bound}
\begin{aligned}
&\frac{1}{R}\sum\limits_{t = 1}^R \mathbb{E}  [{\Vert\nabla _{\bf{w}}}f({{\bf{w}}^{t - 1}})\Vert{^2}] \\ & \le \frac{2\vartheta}{{\gamma R}} \!+\! \frac{{\beta \gamma \sum\limits_{i = 1}^N {\sum\limits_{j = 1}^L {\frac{{\sigma _j^2}}{{{b_i}}}} }  } } {N^2}
 + {\mathbbm{1}}_{\{I > 1\}} 4{\beta ^2}{\gamma ^2}{I^2}\sum\limits_{j = 1}^{{L_c}}\! {G_j^2} ,
\end{aligned}
\end{equation}}%
where $\vartheta  = f({{\bf{w}}^0}) - {f^ * }$, $b_i$ denotes the batch size of the $i$-th edge device, $L$ and $f^{\ast}$ represent the total number of global model layers and global minimum value of problem~\eqref{minimiaze_loss_function}, respectively. 
\end{theorem}

\begin{proof}
We begin by analyzing the expected decrease in the global loss function across one training round. By invoking the smoothness of the loss function $f\left(  \cdot  \right)$, we have
{\begin{align}\label{eq:equal_1_total}
\mathbb{E}[ {f({{\bf{w}}^t})} ] \le& {\rm{ }}\mathbb{E}[ {f({{\bf{w}}^{t - 1}})} ] + \mathbb{E}[ {\langle {\nabla _{\bf{w}}}f({{\bf{w}}^{t - 1}}),{{\bf{w}}^t} - {{\bf{w}}^{t - 1}}\rangle } ]{\rm{  }} \nonumber \\&~+ \frac{\beta }{2}\mathbb{E}[ {{{\| {{{\bf{w}}^t} - {{\bf{w}}^{t - 1}}} \|}^2}} ].
\end{align}
}

Note that
{\begin{align}\label{eq:w_decouple}
 &\mathbb{E}[ {{{\| {{{\bf{w}}^t} - {{\bf{w}}^{t - 1}}} \|}^2}} ]\nonumber\\
 =&\mathbb{E}[ {{{\| {[ {{\bf{h}}_c^t;{\bf{h}}_s^t} ] - [ {{\bf{h}}_c^{t - 1};{\bf{h}}_s^{t - 1}} ]} \|}^2}} ]\nonumber\\ =&\mathbb{E}[ {{{\| {[ {{\bf{h}}_c^t - {\bf{h}}_c^{t - 1};{\bf{h}}_s^t - {\bf{h}}_s^{t - 1}} ]} \|}^2}} ]\nonumber\\=&\mathbb{E}[ {{{\| { {{\bf{h}}_c^t - {\bf{h}}_c^{t - 1}} } \|}^2}} ] + \mathbb{E}[ {{{\| { {{\bf{h}}_s^t - {\bf{h}}_s^{t - 1}} } \|}^2}} ],
\end{align}}%
where $\mathbb{E}[ {{{\| { {{\bf{h}}_c^t - {\bf{h}}_c^{t - 1}} } \|}^2}} ]$ can be bounded, as given by  
{\begin{align}\label{eq:wc_squre}
 &\mathbb{E}[ {{{\| { {{\bf{h}}_c^t - {\bf{h}}_c^{t - 1}} } \|}^2}} ] \overset{(a)}{=}{\gamma ^2}\mathbb{E}[||\frac{1}{N}\sum\limits_{i = 1}^N {{\bf{g}}_{c,i}^t} |{|^2}] \nonumber\\ 
 \overset{(b)}{=}& {\gamma ^2}\mathbb{E}[||\frac{1}{N}\sum\limits_{i = 1}^N {\left( {{\bf{g}}_{c,i}^t - {\nabla _{{{\bf{h}}_c}}}{f_i}\left( {{\bf{h}}_{c,i}^{t - 1}} \right)} \right)} |{|^2}]\nonumber\\ 
 +& {\gamma ^2}\mathbb{E}[||\frac{1}{N}\sum\limits_{i = 1}^N {{\nabla _{{{\bf{h}}_c}}}{f_i}\left( {{\bf{h}}_{c,i}^{t - 1}} \right)} |{|^2}] \nonumber\\
\overset{(c)}{=}& \frac{{{\gamma ^2}}}{{{N^2}}}\sum\limits_{i = 1}^N  \mathbb{E}[||{\bf{g}}_{c,i}^t - {\nabla _{{{\bf{h}}_c}}}{f_i}\left( {{\bf{h}}_{c,i}^{t - 1}} \right)|{|^2}]\nonumber\\
+& {\gamma ^2}\mathbb{E}[||\frac{1}{N}\sum\limits_{i = 1}^N {{\nabla _{{{\bf{h}}_c}}}{f_i}\left( {{\bf{h}}_{c,i}^{t - 1}} \right)} |{|^2}] \nonumber\\
\overset{(d)}{\leq }& \frac{{{\gamma ^2}}}{{{N^2}}}\sum\limits_{i = 1}^N {\sum\limits_{j = 1}^{L_c} {\frac{{\sigma _j^2}}{{{b_i}}}} }  +{\gamma ^2}\mathbb{E}[||\frac{1}{N}\sum\limits_{i = 1}^N {{\nabla _{{{\bf{h}}_c}}}{f_i}\left( {{\bf{h}}_{c,i}^{t - 1}} \right)} |{|^2}],
\end{align}}%
where (a) follows from {Eqn.~\eqref{h_c_define}} and Eqn.~\eqref{lemma_eq_1}; (b) follows that $\mathbb{E}[\mathbf{g}_{c,i}^{t}] = \nabla_{{\bf{h}}_c} f_{i}({\mathbf{h}}_{c,i}^{t-1})$ and  $\mathbb{E}[\Vert \mathbf{z} \Vert^{2}] = \mathbb{E} [ \Vert \mathbf{\mathbf{z}} - \mathbb{E}[\mathbf{z}]\Vert^{2}] + \Vert\mathbb{E}[\mathbf{z}] \Vert^{2}$ that holds for any random vector $\mathbf{z}$; (c) is because $\mathbf{g}_{c,i}^{t} - \nabla_{{\bf{h}}_c} f_{i}(\mathbf{h}_{c,i}^{t-1})$ has zero mean and is independent between edge devices; and (d) follows from {\bf Assumption \ref{asp:2}}.

Similarly, $\mathbb{E}[ {{{\| { {{\bf{h}}_s^t - {\bf{h}}_s^{t - 1}} ]} \|}^2}} $ has an upper bound:
{\begin{align}\label{eq:ws_squre}
&\mathbb{E}[ {{{\| { {{\bf{h}}_s^t - {\bf{h}}_s^{t - 1}} } \|}^2}}] \nonumber\\\le& \frac{{{\gamma ^2}}}{{{N^2}}}\sum\limits_{i = 1}^N \!\!{\sum\limits_{j = {L_c} + 1}^L \!\!{\frac{{\sigma _j^2}}{{{b_i}}}} } \! +\! {\gamma ^2}\mathbb{E}[||\frac{1}{N}\sum\limits_{i = 1}^N {{\nabla _{{{\bf{h}}_s}}}{f_i}\left( {{\bf{h}}_{s,i}^{t - 1}} \right)} |{|^2}].
\end{align}}

Substituting Eqn.~\eqref{eq:wc_squre} and Eqn.~\eqref{eq:ws_squre} into Eqn.~\eqref{eq:w_decouple} yields
{\begin{align}\label{eq:w-difference-squre}
&\mathbb{E}[\|{{\bf{w}}^t} - {{\bf{w}}^{t - 1}}\|{^2}] \nonumber \\\le& \frac{{{\gamma ^2}}}{{{N^2}}}\sum\limits_{i = 1}^N {\sum\limits_{j = 1}^L {\frac{{\sigma _j^2}}{{{b_i}}}} }  + {\gamma ^2}\mathbb{E}[||\frac{1}{N}\sum\limits_{i = 1}^N {{\nabla _{{{\bf{h}}_c}}}{f_i}\left( {{\bf{h}}_{c,i}^{t - 1}} \right)} |{|^2}] \nonumber \\
+& {\gamma ^2}\mathbb{E}[||\frac{1}{N}\sum\limits_{i = 1}^N {{\nabla _{{{\bf{h}}_s}}}{f_i}\left( {{\bf{h}}_{s,i}^{t - 1}} \right)} |{|^2}]\nonumber\\
=& \frac{{{\gamma ^2}}}{{{N^2}}}\sum\limits_{i = 1}^N {\sum\limits_{j = 1}^L {\frac{{\sigma _j^2}}{{{b_i}}}} }  + {\gamma ^2}\mathbb{E}[||\frac{1}{N}\sum\limits_{i = 1}^N {{\nabla _{\bf{w}}}{f_i}\left( {{\bf{w}}_i^{t - 1}} \right)} |{|^2}].
\end{align}}

We further note that 
{\begin{align} \label{eq:inner_product_w}
&\mathbb{E}[\langle \nabla_{\bf{w}} f({\mathbf{w}}^{t-1}), {\mathbf{w}}^{t} - {\mathbf{w}}^{t-1}\rangle] \nonumber\\
\overset{(a)}{=}& -\gamma \mathbb{E} [\langle \nabla_{\bf{w}} f({\mathbf{w}}^{t-1}), \frac{1}{N} \sum_{i=1}^{N} \mathbf{g}_{i}^{t}\rangle] \nonumber \\
\overset{(b)}{=}& -\gamma \mathbb{E}[\langle {\nabla_{\bf{w}}}f({{\bf{w}}^{t - 1}}),\frac{1}{N}\sum\limits_{i = 1}^N \nabla_{\bf{w}}  {f_i}({\bf{w}}_i^{t - 1})\rangle ]\nonumber \\
\overset{(c)}=&  - \frac{\gamma }{2}\mathbb{E}[||{\nabla _{\bf{w}}}f({{\bf{w}}^{t - 1}})|{|^2} + ||\frac{1}{N}\sum\limits_{i = 1}^N {{\nabla _{\bf{w}}}} {f_i}({\bf{w}}_{i}^{t - 1})|{|^2} \nonumber \\
&- ||{\nabla _{{{\bf{w}}}}}f({{\bf{w}}^{t - 1}}) - \frac{1}{N}\sum\limits_{i = 1}^N {{\nabla _{{{\bf{w}}}}}} {f_i}({\bf{w}}_{i}^{t - 1})|{|^2}],
\end{align}}%
where (a) follows from ${{\bf{w}}^t} = \frac{{{1}}}{N} \sum\limits_{i = 1}^N {\bf{w}}_i^t$; (c) follows from the  identity $\langle \mathbf{z}_{1}, \mathbf{z}_{2}\rangle = \frac{1}{2} \big( \Vert \mathbf{z}_{1}\Vert^{2} + \Vert \mathbf{z}_{2}\Vert^{2} - \Vert \mathbf{z}_{1} - \mathbf{z}_{2}\Vert^{2} \big)$ for any two vectors $\mathbf{z}_{1}, \mathbf{z}_{2}$ of the same length; (b) follows from 
{
	\begin{align*}
	&\mathbb{E}[\langle \nabla_{\bf{w}} f({\mathbf{w}}^{t-1}), \frac{1}{N} \sum_{i=1}^{N} \mathbf{g}_{i}^{t}\rangle] \\
	=& \mathbb{E}[\mathbb{E}[\langle \nabla_{\bf{w}} f({\mathbf{w}}^{t-1}), \frac{1}{N} \sum_{i=1}^{N} \mathbf{g}_{i}^{t}\rangle | \boldsymbol{\xi}^{[t-1]}]] \\
	=& \mathbb{E}[\langle \nabla_{\bf{w}} f({\mathbf{w}}^{t-1}), \frac{1}{N} \sum_{i=1}^{N} \mathbb{E}[\mathbf{g}_{i}^{t}| \boldsymbol{\xi}^{[t-1]}]\rangle ]\\
	 =& \mathbb{E}[\langle \nabla_{\bf{w}} f({\mathbf{w}}^{t-1}), \frac{1}{N} \sum_{i=1}^{N} \nabla_{\bf{w}} f_{i}(\mathbf{w}_{i}^{t-1})\rangle ],
	\end{align*}}%
where the first equality follows by the law of expectations, the second equality holds because ${\mathbf{w}}^{t-1}$ is determined by $\boldsymbol{\xi}^{[t-1]}= [\boldsymbol{\xi}^{1}, \ldots, \boldsymbol{\xi}^{t-1}]$, and the third equality follows from $\mathbb{E}[\mathbf{g}_{i}^{t} | \boldsymbol{\xi}^{[t-1]}] = \mathbb{E}[\nabla F_{i}(\mathbf{w}_{i}^{t-1};\xi^{t}_{i}) | \boldsymbol{\xi}^{[t-1]}] = \nabla f_{i}(\mathbf{w}_{i}^{t-1})$.

Substituting Eqn.~\eqref{eq:w-difference-squre} and Eqn.~\eqref{eq:inner_product_w} into Eqn.~\eqref{eq:equal_1_total}, we have  
{\begin{align} \label{eq:11}
&\mathbb{E}[f({\mathbf{w}}^{t})] \nonumber\\
\leq &\mathbb{E}[f({\mathbf{w}}^{t-1})] - \frac{\gamma - \gamma^{2}\beta}{2} \mathbb{E} [\Vert \frac{1}{N} \sum_{i=1}^{N} \nabla_{\bf{w}} f_{i} (\mathbf{w}_{i}^{t-1})\Vert^{2}] \nonumber \\
 &- \frac{\gamma}{2} \mathbb{E}[\Vert \nabla_{\bf{w}} f({\mathbf{w}}^{t-1})\Vert^{2}+ \frac{{\beta {\gamma ^2}}}{{2{N^2}}}\sum\limits_{i = 1}^N {\sum\limits_{j = 1}^L {\frac{{\sigma _j^2}}{{{b_i}}}} }   \nonumber\\ &+ \frac{\gamma}{2}\mathbb{E}[||{\nabla _{{{\bf{w}}}}}f({{\bf{w}}^{t - 1}}) -\frac{1}{N}\sum\limits_{i = 1}^N {{\nabla _{{{\bf{w}}}}}} {f_i}({\bf{w}}_i^{t - 1})|{|^2}]\nonumber \\
\overset{(a)}{\leq}& \mathbb{E}[f({\mathbf{w}}^{t-1})] - \frac{\gamma}{2} \mathbb{E}[\Vert \nabla_{\bf{w}} f({\mathbf{w}}^{t-1})\Vert^{2}+ \frac{{\beta {\gamma ^2}}}{{2{N^2}}}\sum\limits_{i = 1}^N {\sum\limits_{j = 1}^L {\frac{{\sigma _j^2}}{{{b_i}}}} }  \nonumber\\ &+ \frac{\gamma}{2}\mathbb{E}[||{\nabla _{{{\bf{h}}_c}}}f({{\bf{h}}_c^{t - 1}}) -\frac{1}{N}\sum\limits_{i = 1}^N {{\nabla _{{{\bf{h}}_c}}}} {f_i}({\bf{h}}_{c,i}^{t - 1})|{|^2}]\nonumber \\
&+\frac{\gamma}{2}\mathbb{E}[||{\nabla _{{{\bf{h}}_s}}}f({{\bf{h}}_s^{t - 1}}) - \frac{1}{N}\sum\limits_{i = 1}^N {{\nabla _{{{\bf{h}}_s}}}} {f_i}({\bf{h}}_{s,i}^{t - 1})|{|^2}] \nonumber\\
\overset{(b)}{\leq} &\mathbb{E}[f({\mathbf{w}}^{t-1})] - \frac{\gamma}{2} \mathbb{E}[\Vert \nabla_{\bf{w}} f({\mathbf{w}}^{t-1})\Vert^{2}]+ \frac{{\beta {\gamma ^2}}}{{2{N^2}}}\sum\limits_{i = 1}^N {\sum\limits_{j = 1}^L {\frac{{\sigma _j^2}}{{{b_i}}}} }  \nonumber\\ &+ {\mathbbm{1}}_{\{I > 1\}} 2\beta^2\gamma^{3} I^{2} \sum\limits_{j = 1}^{L_c} {G_j^2}, 
\end{align}}%
where (a) follows from $0 < \gamma \leq \frac{1}{\beta}$ and (b) holds because of the following inequality~\eqref{difference_wc} and~\eqref{difference_ws}
{\begin{align}\label{difference_wc}
&\mathbb{E}[ \Vert \nabla_{{{\bf{h}}_c}} f({\mathbf{h}}_c^{t-1}) - \frac{1}{N} \sum_{i=1}^{N} \nabla_{{{\bf{h}}_c}} f_{i} (\mathbf{h}_{c,i}^{t-1})\Vert^{2}] \nonumber \\
 =& \mathbb{E} [ \Vert \frac{1}{N} \sum_{i=1}^{N}\nabla_{{{\bf{h}}_c}} f_{i}({\mathbf{h}}_c^{t-1}) - \frac{1}{N} \sum_{i=1}^{N} \nabla_{{{\bf{h}}_c}} f_{i} (\mathbf{h}_{c,i}^{t-1})\Vert^{2}] \nonumber \\
=& \frac{1}{N^{2}} \mathbb{E} [\Vert \sum_{i=1}^{N} \big( \nabla_{{{\bf{h}}_c}} f_{i}({\mathbf{h}}_c^{t-1}) - \nabla f_{{\mathbf{h}}_c} (\mathbf{h}_{c,i}^{t-1}) \big)\Vert^{2}] \nonumber \\
\leq& \frac{1}{N} \mathbb{E} [ \sum_{i=1}^{N} \Vert \nabla_{{{\bf{h}}_c}} f_{i}({\mathbf{h}}_c^{t-1}) - \nabla_{{{\bf{h}}_c}} f_{i} (\mathbf{h}_{c,i}^{t-1})\Vert ^{2}] \nonumber \\
\leq& \beta^2\frac{1}{N} \sum_{i=1}^{N}\mathbb{E}[ \Vert {\mathbf{h}}_c^{t-1} - \mathbf{h}_{c,i}^{t-1}\Vert^{2}] \nonumber \\
\leq&  {\mathbbm{1}}_{\{I > 1\}} 4\beta^2\gamma^{2} I^{2} \sum\limits_{j = 1}^{L_c} {G_j^2},
\end{align}}%
where the first inequality follows from $\Vert \sum_{i=1}^{N} \mathbf{z}_{i}\Vert^{2} \leq N \sum_{i=1}^{N} \Vert \mathbf{z}_{i}\Vert^{2}$ for any vectors $\mathbf{z}_{i}$; the second inequality follows from the smoothness of $f_{i}$ under {\bf Assumption \ref{asp:1}}; and the third inequality follows from {\bf Lemma \ref{lm:diff-avg-per-node}}. Moreover, we have
{ \begin{align}\label{difference_ws}
&\mathbb{E}[ \Vert \nabla_{{{\bf{h}}_s}} f({\mathbf{h}}_s^{t-1}) - \frac{1}{N} \sum_{i=1}^{N} \nabla_{{{\bf{h}}_s}} f_{i} (\mathbf{h}_{s,i}^{t-1})\Vert^{2}] \nonumber \\
\leq& \beta^2\frac{1}{N} \sum_{i=1}^{N}\mathbb{E}[ \Vert {\mathbf{h}}_s^{t-1} - \mathbf{h}_{s,i}^{t-1}\Vert^{2}] \overset{(a)}{=}0, 
\end{align}
where (a) holds because the server-side common sub-models are aggregated in each training round (i.e., $I = 1$). Therefore, at any training round $t$, the server-side common sub-model of each edge device is the aggregated version.

Dividing both sides of Eqn.~\eqref{eq:11} by $\frac{\gamma}{2}$ and rearranging terms yields
{ \footnotesize \begin{align}
&\mathbb{E}\left [\Vert \nabla_{\bf{w}} f({\mathbf{w}}^{t-1})\Vert^{2}\right] \nonumber \\
\leq& \frac{2}{\gamma}\! \left(\mathbb{E}\left[f({\mathbf{w}}^{t-1})\right] \! - \!\mathbb{E}\left[f({\mathbf{w}}^{t})\right]\right) \!+\!\frac{\beta\gamma \!\sum\limits_{i = 1}^N {\sum\limits_{j = 1}^L {\frac{{\sigma _j^2}}{{{b_i}}}} }  }{N^2}  \nonumber \!+\! {\mathbbm{1}}_{\{I > 1\}} 4\beta^2\gamma^{2} I^{2}\! \sum\limits_{j = 1}^{L_c} \! {G_j^2}. \label{eq:pf-thm-rate-eq8}
\end{align}
}%
Summing over $t\in\{1,\ldots, R\}$ and dividing both sides by $R$ yields
{\footnotesize 
\begin{align*}
&\frac{1}{R} \sum_{t=1}^{R} \mathbb{E}\left [\Vert \nabla_{\bf{w}} f({\mathbf{w}}^{t-1})\Vert^{2}\right] \\
\leq &\frac{2}{\gamma R} \!\left(f({\mathbf{w}}^{0}\!) - \mathbb{E}\left[f({\mathbf{w}}^{R})\right]\right) \!+\!\frac{\beta\gamma\!\! \sum\limits_{i = 1}^N \!{\sum\limits_{j = 1}^L \!\!{\frac{{\sigma _j^2}}{{{b_i}}}} }  }{N^2}  \nonumber \!+\! {\mathbbm{1}}_{\{I > 1\}} 4\beta^2\gamma^{2} I^{2} \!\sum\limits_{j = 1}^{L_c} \!{G_j^2} \\
\overset{(a)}{\leq} &  \!\frac{2}{\gamma R}\! \left(f({\mathbf{w}}^{0}) \!-\!f^{\ast}\right) \!+\!\frac{\beta\gamma \sum\limits_{i = 1}^N {\sum\limits_{j = 1}^L {\frac{{\sigma _j^2}}{{{b_i}}}} }  }{N^2}  \nonumber \!+\! {\mathbbm{1}}_{\{I > 1\}} 4\beta^2\gamma^{2} I^{2} \!\sum\limits_{j = 1}^{L_c} {G_j^2},
\end{align*}}
where (a) is because $f^{\ast}$ is the global minimum value of problem {\eqref{minimiaze_loss_function}}.}
\end{proof}

Substituting Eqn.~\eqref{convergence_bound} into Eqn.~\eqref{accuracy_cons_corollary} yields {\bf{Corollary 1}}, revealing a lower bound on the number of training rounds for achieving target convergence accuracy.
\begin{corollary}\label{theorem1}
The number $R$ of {training rounds} for achieving target convergence accuracy $\varepsilon$, i.e., satisfying
{ \begin{equation}\label{accuracy_cons_corollary}
\frac{1}{R} \sum_{t=1}^{R} \mathbb{E} [\Vert \nabla_{\bf{w}} f({\mathbf{w}}^{t-1})\Vert^{2}] \le \varepsilon,
\end{equation}}
is lower bounded by
{ \begin{equation}\label{lowest_com_num}
R  \ge \frac{2 \vartheta}{{\gamma \bigg( {\varepsilon  - \frac{{\beta \gamma \sum\limits_{i = 1}^N {\sum\limits_{j = 1}^L {\frac{{\sigma _j^2}}{{{b_i}}}} }  } }{N^2} - {\mathbbm{1}}_{\{I > 1\}} 4{\beta ^2}{\gamma ^2}{I^2}\sum\limits_{j = 1}^{{L_c}} {G_j^2}} \bigg)}}.
\end{equation}}
\end{corollary}
{\textbf{Insight 1:}  Eqn.~\eqref{lowest_com_num} reveals that the number $R$ of training rounds for achieving target convergence accuracy $\varepsilon$ decreases with the increase of $b_i$. This is because a larger $b_i$ reduces the variance of stochastic gradients, thereby enhancing the training convergence. For a fixed number of training rounds $R$, increasing $b_i$ leads to higher converged accuracy, i.e., smaller $\varepsilon$. It is interesting to note that the BSs of clients can compensate for each other: A stronger client can take a larger BS, whereas a weaker client takes a smaller BS, thereby mitigating the straggler effect without affecting the convergence upper bound.}

{\textbf{Insight 2:}  Eqn.~\eqref{lowest_com_num} also shows that $L_c$ plays a critical role in model convergence when the client-side aggregation interval $I>1$. A smaller $L_c$ (i.e., placing more layers on the edge server) enables more frequent aggregations of a larger portion of the model for expediting convergence, whereas a larger $L_c$ retains more layers on the edge devices, slowing down model convergence. Conversely, when $I=1$, $L_c$ has no impact on convergence since all layers are aggregated synchronously.


These observations align with the experimental results described in Fig.~\ref{fig:motivation_1} and Fig.~\ref{fig:motivation_2}. While a larger BS can expedite training convergence by reducing gradient variance, it also incurs higher per-round computation and communication costs, which may prolong training latency on resource-constrained edge devices. On the other hand, MS determines not only the computing workload between edge devices and the server, but also the volume of smashed data exchanged. Therefore, optimizing BS and MS, i.e., assigning different BS and MS to diverse edge devices, is essential for accelerating SFL at heterogeneous edge devices.


\section{Problem Formulation}\label{prob_formu}

The convergence analysis quantifies the impact of BS and MS on model convergence. Building upon the convergence upper bound in Section~\ref{convergence_HASFL}, we formulate a joint BS and MS optimization problem to minimize the training latency of achieving convergence for HASFL over edge networks with heterogeneous participating edge devices. Then, we propose a heterogeneity-aware BS and MS strategy in Section~\ref{solu_appro}. For clarity, the decision variables and definitions are listed below. 
\begin{itemize}
\item ${\boldsymbol b}$: $b_i \in {\mathbb{N}^ + }$ represents the BS decision, which indicates the number of data samples utilized by the $i$-th edge device for one training round. ${{\bf b} } = [b_1 ,b_2 ,...,b_N]$ denotes the collection of BS decisions.
\item ${\boldsymbol\mu}$: $\mu_{i,j} \in \left\{ {0,1} \right\}$ is the MS decision, where $\mu_{i,j}=1$ indicates that the $j$-th neural network layer is selected as the cut layer for the $i$-th edge device, and $\mu_{i,j}=0$, otherwise.
${\boldsymbol\mu}  = \left[ {{\mu _{1,1}},{\mu _{1,2}},...,{\mu _{N,L}}} \right]$ represents the collection of MS decisions.
\end{itemize}

\subsection{Training Latency Analysis}

This section presents a detailed analysis for the training latency of HASFL. Without loss of generality, we focus on one training round for analysis. For notational brevity, the training round number index $t$ is omitted. We start by analyzing the latency of the split training stage for one training round as follows.

\textit{a1) Client-side model forward propagation latency:}
In this step, each edge device performs the client-side FP with a mini-batch sampled from its local dataset. The computing workload (in float-point operations (FLOPs)~\cite{lin2023efficient,yi2024qsfl,lin2024adaptsfl}) of the client-side FP for the $i$-th edge device per data sample is represented as $\Phi_{c,i}^{F}\! \left( \boldsymbol{\mu}  \right) = \sum\limits_{j = 1}^{L} {{\mu _{i,j}}} {\rho _j}$, where ${\rho _j}$ denotes the FP computing workload of propagating the first $j$ layer of neural network per data sample. The $i$-th edge device utilizes a mini-batch containing $b_i$ data samples to execute the client-side FP. The client-side FP latency of the $i$-th edge device is given by~\cite{wu2023split,lin2023efficient}
\begin{align}\label{client_FP_latency}
T_i^{F} = \frac{{b_i\,\Phi_{c,i}^{F}\! \left( \boldsymbol{\mu}  \right)}}{{{f_i}}} ,\;\forall i \in \mathcal{N},
\end{align}
where ${{f _i}}$ denotes the computing capability (in float-point operations per second (FLOPS)) of the $i$-th edge device.

\textit{ a2) Activation uploading latency:}
After the client-side FP is completed, each edge device transmits the activations to the edge server. Let $\Gamma_{a,i} \left( \boldsymbol{\mu}  \right) = \sum\limits_{j = 1}^{L} {{\mu _{i,j}}} {\psi _j}$ represent the data size (in bits) of activations for the $i$-th edge device, where ${\psi _j}$ denotes the data size of activations at the cut layer $j$. For the $i$-th edge device, the activation uploading latency is given by
\begin{align}\label{smashed_trans_latency}
T_{a,i}^{U} = \frac{{b_i\Gamma_{a,i} \left( \boldsymbol{\mu}  \right)}}{{r_i^{U}}},\;\forall i \in \mathcal{N},
\end{align}
where ${r_i^{U}}$ represents the uplink data rate from the $i$-th edge device to the edge server.

\textit{a3) Server-side model forward propagation and backward pass latency:}
In this step, the edge server performs server-side FP and BP with activations received from all participating edge devices. Let $\Phi _s^F\left( {\boldsymbol{b}, \boldsymbol{\mu }} \right) = \sum\limits_{i = 1}^N {\sum\limits_{j = 1}^L {b_i {\mu _{i,j}}} \left( {{\rho _L} - {\rho _j}} \right)} $ and  $\Phi _s^{B}\left( \boldsymbol{b}, \boldsymbol{\mu}  \right) = \sum\limits_{i = 1}^N {\sum\limits_{j = 1}^{L} {b_i {\mu _{i,j}}} \left( {{\varpi_L} - {\varpi _j}} \right)}$ denote the computing workload of the server-side FP and BP, respectively, where ${\varpi _j}$ denotes the BP computing workload of propagating the first $j$ layer of neural network per data sample. Thus, the server-side model FP and BP latency can be determined as
\begin{align}\label{server_FP_latency}
T_s^{F} = \frac{{\Phi _s^{F}\left( \boldsymbol{b}, \boldsymbol{\mu}  \right)}}{{{f_s}}}
\end{align}
and 
\begin{align}\label{server_BP_latency}
T_s^{B} = \frac{{\Phi _s^{B}\left( \boldsymbol{b}, \boldsymbol{\mu}  \right)}}{{{f_s}}},
\end{align}
where ${{f _s}}$ is the computing capability of the edge server.

\textit{a4) Downloading latency of activations’ gradients:}
After the completion of server-side BP, the edge server sends the activations' gradients back to the respective edge devices. Let $\Gamma_{g,i} \left( \boldsymbol{\mu}  \right) = \sum\limits_{j = 1}^{L} {{\mu _{i,j}}} {\chi _j}$ represent the data size of activations' gradients for the $i$-th edge device per data sample, where ${\chi _j}$ is the data size of activations' gradients at the cut layer $j$. Downloading latency of activations’ gradients of the $i$-th edge device can be calculated as
\begin{align}\label{downlink_latency}
T_{g,i}^D = \frac{{ {b_i } {\Gamma _{g,i}}\left( {\boldsymbol{\mu} } \right)}}{{r_i^D}},\;\forall i \in {\cal N},
\end{align}
where ${r_i^{D}}$ is the downlink data rate from edge server to the $i$-th edge device.

\textit{a5) Client-side model backward pass latency:}
In this step, each edge device executes client-side BP based on the received activations' gradients. Let $\Phi _{c,i}^{B}\left( \boldsymbol{\mu}  \right) = \sum\limits_{j = 1}^{L} {{\mu _{i,j}}} {\varpi _j}$ represent the computing workload of the client-side BP for the $i$-th edge device  per data sample. For the $i$-th edge device, the client-side BP latency can be obtained as
\begin{align}\label{client_BP_latency}
T_i^{B} = \frac{{b_i\, \Phi _{c,i}^{B}\left( \boldsymbol{\mu}  \right)}}{{{f_i}}},\;\forall i \in \mathcal{N}.
\end{align}

Next, we analyze the latency of the client-side model aggregation stage every $I$ training rounds as follows.

\textit{b1) Sub-model uploading latency:} In this step, each edge device sends its client-side sub-model to the fed server, while the edge server uploads the server-side non-common sub-models to the fed server. Let $\Lambda_{c,i} \left( \boldsymbol{\mu}  \right) = \sum\limits_{j = 1}^{L} {{\mu _{i,j}}} {\delta _j}$ denote the data size of client-side sub-model for the $i$-th edge device, where ${\delta _j}$ is the data size of client-side sub-model with the cut layer $j$. The total data size of the exchanged server-side non-common sub-models is denoted by ${\Lambda _{s}}\left( {\boldsymbol{\mu }} \right) = N \mathop {\max }\limits_i \left\{ { \sum\limits_{j = 1}^L {{\mu _{i,j}}} {\delta _j}} \right\} - \sum\limits_{i = 1}^N \sum\limits_{j = 1}^L {{\mu _{i,j}}} {\delta _j}$.  The uploading latency of client-side sub-model for the $i$-th edge device and server-side non-common sub-models are expressed as 
\begin{align}\label{smashed_trans_latency}
T_{c,i}^U = \frac{{{\Lambda _{c,i}}\left( {\boldsymbol{\mu }} \right)}}{{r_{i,f}^U}}, \;\forall i \in \mathcal{N}
\end{align}
and 
\begin{align}\label{smashed_trans_latency}
T_s^U = \frac{{{\Lambda _{s}}\left( {\boldsymbol{\mu }} \right)}}{{r_{s,f}}},
\end{align}
where $r_{i,f}^U$ and $r_{s,f}$ denote the uplink data rate for transferring the sub-model from the $i$-th edge device and edge server to the fed server, respectively.

\textit{b2) Client-side model aggregation:} The fed server assembles the received client-side sub-models and server-side non-common sub-models into client-specific models and then aggregates them. For simplicity, the
sub-model aggregation latency for this part is ignored, as it is
negligible compared to the latency from other stages~\cite{shi2020joint,xia2021federated}.

\textit{b3) Sub-model downloading latency:} After completing the client-side model aggregation, the fed server sends the updated client-side sub-models to the respective edge devices, along with the server-side non-common sub-models to the edge server. The downloading latency of the client-side sub-model for the $i$-th edge device and server-side non-common sub-models are given by
\begin{align}\label{smashed_trans_latency}
T_{c,i}^D = \frac{{{\Lambda _{c,i}}\left( {\boldsymbol{\mu }} \right)}}{{r_{i,f}^D}}, \;\forall i \in \mathcal{N}
\end{align}
and 
\begin{align}\label{smashed_trans_latency}
T_s^D = \frac{{{\Lambda _{s}}\left( {\boldsymbol{\mu }} \right)}}{{r_{f,s}}},
\end{align}
where $r_{i,f}^D$ and $r_{f,s}$ represent the downlink data rates for transmitting sub-models from the fed server to the $i$-th edge device and edge server, respectively.

\subsection{Problem Formulation}

\begin{figure}[t]
\vspace{-.5ex}
\setlength\abovecaptionskip{3pt}
\centering
\subfigure[Split training.]{
    \includegraphics[height=2.7cm,width=4.3cm]{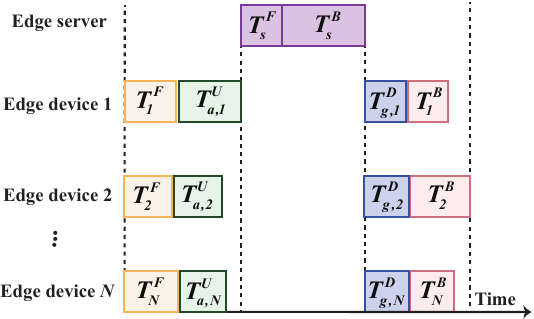}
    \label{sfig:split_training}
}
\subfigure[Client-side model aggregatio.]{
    \includegraphics[height=2.7cm,width=3cm]{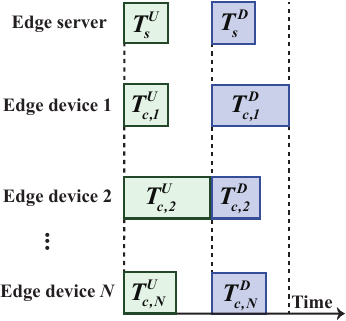}
    \label{sfig:client_side_aggrega}
}
    \caption{An illustration of split training and client-side model aggregation stages.}
    \label{fig:training_2_stage}
    \vspace{-2ex}
\end{figure}

As illustrated in Fig.~\ref{fig:training_2_stage}, the per-round latency of split training can be expressed as
{{ \begin{align}\label{split_training_latency}
T_S\!\left( {\boldsymbol{b}, \boldsymbol{\mu} } \right) &\!=\! \!{\mathop {\max }\limits_i\! \left\{ {T_i^F \!+\! T_{a,i}^U} \right\} \!\!+\! T_s^F \!+\! T_s^B \!+\! \mathop {\max }\limits_i \left\{ {T_{g,i}^D + T_i^B} \right\}},
\end{align}}}
and client-side model aggregation latency is calculated as
{{ \begin{align}\label{model_aggregation_latenvy}
T_A\!\left( {\boldsymbol{b}, \boldsymbol{\mu} } \right) \!= \mathop {\max }\limits_i \left\{ {T_{c,i}^U,T_s^U} \right\} + \mathop {\max }\limits_i \left\{ {T_{c,i}^D,T_s^D} \right\}.
\end{align}}}

Since split training is conducted in each training round and client-side model aggregation occurs every $I$ training rounds, the total latency for $R$ training rounds is given by
{{ \begin{align}\label{total_latency}
T\!\left( {\boldsymbol{b}, \boldsymbol{\mu} } \right) =RT_S\!\left( {\boldsymbol{b}, \boldsymbol{\mu} } \right) + {\left\lfloor {\frac{R}{{{I}}}} \right\rfloor T_{A}\!\left( {\boldsymbol{b}, \boldsymbol{\mu} } \right)}. 
\end{align}}}

As alluded in Section~\ref{Intro}, BS balances the tradeoffs between per-round latency and training convergence, while MS significantly impacts communication-computing overhead and model convergence. Consequently, a joint optimization of BS and MS is crucial for expediting the training process. To this end, we formulate the following problem to minimize the training latency for model convergence:
\begin{align}\label{time_minimize_problem}
\mathcal{P}:&\mathop {{\rm{min}}}\limits_{{\boldsymbol{b}},{\boldsymbol{\mu}}}  T\!\left( {\boldsymbol{b}, \boldsymbol{\mu} } \right)  \\
&\mathrm{s.t.} ~\mathrm{C1:}~\frac{1}{R}\sum\limits_{t = 1}^R \mathbb{E} [{\Vert\nabla _{\bf{w}}}f({{\bf{w}}^{t - 1}})\Vert{^2}] \le \varepsilon, \nonumber \\
&~\mathrm{C2:}~{\mu _{i,j}} \in \left\{ {0,1} \right\}, \quad\forall i \in \mathcal{N}, \; j = 1,2,...,L, \nonumber\\
&~\mathrm{C3:}~\sum\limits_{j = 1}^{L} {{\mu _{i,j}}}  = 1, \quad \forall i \in \mathcal{N}, \nonumber \\
&~\mathrm{C4:}~\sum\limits_{j = 1}^L {{{\mu _{i,j}}} } \big( {{b_i {\widetilde \psi} _j} +{b_i {\widetilde \chi} _j}+{{\widetilde \vartheta} _j} + {\delta _j}} \big) < {\upsilon_{c, i}}, \quad \forall i \in \mathcal{N}, \nonumber\\
&~\mathrm{C5:}~ b_i \geq 1, b_i \in \mathbb{Z}, \quad \forall i \in \mathcal{N}, \nonumber 
\end{align}
where ${{\widetilde \psi} _j} = \sum\limits_{k = 1}^j {{\psi _k}} $ and ${{\widetilde \chi} _l} = \sum\limits_{k = 1}^j {{{\chi}_k}}$ are the cumulative sum of data size (in bits) of activations and activations' gradients for the first $j$ layers of the neural network per data sample; ${{\widetilde \vartheta} _l}$ represents the accumulated data size of the optimizer state for the first $j$ layers of the neural network, depending on the choice of the optimizer (e.g. Momentum, SGD, and Adam); $\upsilon_{c, i}$ denotes the memory limitation of the $i$-th edge device. Constrain $\mathrm C1$ guarantees model convergence loss; $\mathrm C2$ and $\mathrm C3$ ensure
the uniqueness of the cut layer for each edge device, indicating that the global model is partitioned into client-side and server-side sub-models;  $\mathrm C4$ represents the memory limitation of edge devices~\cite{yeung2021horus}; $\mathrm C5$ ensures that the BS decision variable is a positive integer.

Problem~\eqref{time_minimize_problem} is a combinatorial problem with a non-convex mixed-integer non-linear objective function, which is generally NP-hard. Consequently, obtaining an optimal solution via polynomial-time algorithms is impractical.

\section{Solution Approach}\label{solu_appro}

In this section, we design an efficient iterative algorithm by decoupling problem~\eqref{time_minimize_problem} into two tractable BS and MS sub-problems and then solving them alternately.

We first derive the explicit expression of $R$ by utilizing {\bf{Corollary 1}}. Given that $R$ is proportional to the objective function, the objective function is minimized if and only if~\eqref{lowest_com_num} holds as an equality. In general, $R$ is substantially larger than $I$, and therefore we can approximate $\left\lfloor \frac{R}{I} \right\rfloor \approx \frac{R}{I}$ in Eqn.~\eqref{total_latency}. By substituting~\eqref{lowest_com_num} into~\eqref{total_latency}, problem~\eqref{time_minimize_problem} can be converted into
\begin{align}\label{problem_1}
\mathcal{P'}:&\mathop {{\rm{min}}}\limits_{{\boldsymbol{b}}, {\boldsymbol{\mu}}} \Theta ( \boldsymbol{b}, {\boldsymbol{\mu }} )   \\
&\mathrm{s.t.} ~\mathrm{C2}-\mathrm{C4}, \nonumber
\end{align}
where 
\begin{align} \label{optim_objective_1}
\small
\Theta ( {\boldsymbol{b}, \boldsymbol{\mu }} )= \frac{{2\vartheta \left( {{T_S}\left( {\boldsymbol{b}, \boldsymbol{\mu }} \right) + \frac{{{T_A}\left( {\boldsymbol{b}, \boldsymbol{\mu }} \right)}}{I}} \right)}} {{\gamma \bigg( \!{\varepsilon  - \frac{{\beta \gamma \!\sum\limits_{i = 1}^N \!{\sum\limits_{j = 1}^L {\frac{{\sigma _j^2}}{{{b_i}}}} }  } }{N^2}\! - \!{\mathbbm{1}}_{\{I > 1\}} 4{\beta ^2}{\gamma ^2}{I^2}\!\sum\limits_{j = 1}^{{L_c}}\! {G_j^2}}\! \bigg)}}. 
\end{align}

Problem~\eqref{problem_1} is a mixed-integer non-linear programming problem, and the non-convexity and non-smoothness of the objective function renders it still intractable. To tackle this issue, we define a set of constants ${\bf{\widetilde G}} = \left[ {{{\widetilde G}^2_1},{{\widetilde G}^2_2},...,{{\widetilde G}^2_L}} \right]$, where ${\widetilde G}^2_j$ denotes the cumulative sum of the bounded second-order moments for the first $j$ layers of neural network, defined as ${{\widetilde G}^2_j} = \sum\limits_{k = 1}^j {G_k^2}$. Thus, the term $\sum\limits_{j = 1}^{{L_c}} {G_j^2}$ in the objective function~\eqref{optim_objective_1} can be reformulated as $\mathop {\max }\limits_i \left\{ {\sum\limits_{j = 1}^L {{\mu _{i,j}}{\widetilde G}_j^2} } \right\} $. To linearize the objective function, we introduce a set of auxiliary variables ${\bf{T}} = \left[ {{T_1},{T_2},...,{T_6}} \right]$, i.e., $\mathop {\max }\limits_i \big \{ {\sum\limits_{j = 1}^L {{\mu _{i,j}}\widetilde G_j^2} } \big\} \le {T_1}$, $\mathop {\max }\limits_i \big\{ {\sum\limits_{j = 1}^L {{\mu _{i,j}}} {\delta _j}} \big\} \le {T_2}$, $\mathop {\max }\limits_i \left\{ {T_i^F + T_{a,i}^U} \right\} \le {T_3}$, $\mathop {\max }\limits_i \left\{ {T_{g,i}^D + T_i^B} \right\} \le {T_4}$, $\mathop {\max }\limits_i \left\{ {T_{c,i}^U,T_s^U} \right\} \le {T_5}$, and $\mathop {\max }\limits_i \left\{ {T_{c,i}^D,T_s^D} \right\} \le {T_6}$. As a consequence, problem~\eqref{problem_1} can be transformed into
{ \begin{align}\label{problem_2}
\mathcal{P''}:&\mathop {{\rm{min}}}\limits_{{\boldsymbol{b}}, {\boldsymbol{\mu}}, \bf{T}, } \Theta' ( {\boldsymbol{b}}, {\boldsymbol{\mu }}, \bf{T} )   \\
&\mathrm{s.t.} ~\mathrm{C2}-\mathrm{C5}, \nonumber\\
&~\mathrm{R1:}~\sum\limits_{j = 1}^L {{\mu _{i,j}}\widetilde G_j^2}  \le {T_1}, \quad\forall i \in \mathcal{N}, \nonumber\\
&~\mathrm{R2:}~\sum\limits_{j = 1}^L {{\mu _{i,j}}} {\delta _j} \le {T_2}, \quad\forall i \in \mathcal{N}, \nonumber\\
&~\mathrm{R3:}~\frac{{b_i \sum\limits_{j = 1}^L {{\mu _{i,j}}} {\rho _j}}}{{{f_i}}} + \frac{{b_i\sum\limits_{j = 1}^L {{\mu _{i,j}}} {\psi _j}}}{{r_i^U}} \le {T_3}, \quad\forall i \in \mathcal{N}, \nonumber\\
&~\mathrm{R4:}~\frac{{b_i\sum\limits_{j = 1}^L {{\mu _{i,j}}} {\chi _j}}}{{r_i^D}} + \frac{{b_i \sum\limits_{j = 1}^L {{\mu _{i,j}}} {\varpi _j}}}{{{f_i}}} \le {T_4}, \quad\forall i \in \mathcal{N},  \nonumber\\
&~\mathrm{R5:}~\frac{{\sum\limits_{j = 1}^L {{\mu _{i,j}}} {\delta _j}}}{{r_{i,f}^U}} \le {T_5}, \quad\forall i \in \mathcal{N},  \nonumber\\
&~\mathrm{R6:}~\frac{{N{T_2} - \sum\limits_{i = 1}^N {\sum\limits_{j = 1}^L {{\mu _{i,j}}} } {\delta _j}}}{{r_{s,f}}} \le {T_5}, \nonumber\\
&~\mathrm{R7:}~\frac{{\sum\limits_{j = 1}^L {{\mu _{i,j}}} {\delta _j}}}{{r_{i,f}^D}} \le {T_6} \quad\forall i \in \mathcal{N},  \nonumber\\
&~\mathrm{R8:}~\frac{{N{T_2} - \sum\limits_{i = 1}^N {\sum\limits_{j = 1}^L {{\mu _{i,j}}} } {\delta _j}}}{{r_{f,s}}} \le {T_6},  \nonumber
\end{align}}
where
{ \begin{align}
\Theta' ( {\boldsymbol{b}}, {\boldsymbol{\mu }}, {\bf{T}} ) =\frac{{2\vartheta \left({ {{T_3} + T_s^F + T_s^B + {T_4}}  + {\frac{{{T_5} + {T_6}}}{I}}}\right) }}{{\gamma \bigg( {\varepsilon  - \frac{{\beta \gamma \!\sum\limits_{i = 1}^N \!{\sum\limits_{j = 1}^L {\frac{{\sigma _j^2}}{{{b_i}}}} }  } }{N^2} - {\mathbbm{1}}_{\{I > 1\}} 4{\beta ^2}{\gamma ^2}{I^2}{T_1}} \bigg)}}. 
\end{align}}

As seen from problem~\eqref{problem_2}, the introduced auxiliary variables are tightly coupled with the original decision variables, posing significant barriers for solving this problem. To combat this, we decompose  problem~\eqref{problem_2} into two tractable sub-problems according to decision variables and devise efficient algorithms to solve them iteratively.

{\textit{BS sub-problem.}} We fix the variables $\boldsymbol{\mu}$ and $\bf{T}$ to investigate the sub-problem involving BS, which is expressed as
\begin{align}\label{sub-problem_1}
\mathcal{P}_1:&\mathop {{\rm{min}}}\limits_{{\boldsymbol{b}}} \Theta' ( {\boldsymbol{b}} )  \\
&\mathrm{s.t.} ~\mathrm{C4}-\mathrm{C5},~\mathrm{R3}-\mathrm{R4}. \nonumber
\end{align}
Then, we can derive the following proposition:
\begin{proposition}\label{theorem2}
The optimal BS decision is given by
{\begin{align}\label{accuracy_cons}
{{\boldsymbol{b}}^*} = \left[ {b_1^*,b_2^*,...,b_N^*} \right],
\end{align}}
where
{\begin{align}\label{b_three_solution}
b_i^* = \left\{ {\begin{array}{*{20}{c}}
1&{{{\hat b}_i} \le 1}\\
{\arg {{\min }_{{b_i} \in \left\{ {\left\lfloor {{{\hat b}_i}} \right\rfloor ,\left\lceil {{{\hat b}_i}} \right\rceil } \right\}}}\Theta '\left( {\boldsymbol{b}} \right)}&{1 < {{\hat b}_i} < {\kappa _i}}\\
{\left\lfloor {{\kappa _i}} \right\rfloor}&{{{\hat b}_i} \ge {\kappa _i}},
\end{array}} \right.
\end{align}}
where \resizebox{0.90\hsize}{!}{$
\kappa_i \! =\! \min \!\left\{\!\!
\frac{ \upsilon_{c,i} - \sum\limits_{j=1}^{L} \mu_{i,j} ( \vartheta_j + \delta_j ) }
     { \sum\limits_{j=1}^{L} \mu_{i,j} ( \psi_j + \chi_j ) },\;\!\!
\frac{ T_3 }
     { \sum\limits_{j=1}^{L} \mu_{i,j} \left( \frac{ \rho_j }{ f_i } + \frac{ \psi_j }{ r_i^U } \right) },\;\!
\frac{ T_4 }
     { \sum\limits_{j=1}^{L} \mu_{i,j} \left( \frac{ \chi_j }{ r_i^D } + \frac{ \varpi_j }{ f_i } \right) }
\!\!\right\}
$}, ${{\hat{\boldsymbol{b}}}} = \{ {{\hat b}_i}\mid i \in {\cal N}\}$ can be easily obtained by solving $\frac{{\partial \Theta '({\boldsymbol{b}})}}{{\partial {\boldsymbol{b}}}} = {\bf{0}}$ with Newton-Jacobi  method~\cite{geng2019pipeline}; $\left\lfloor {\cdot} \right\rfloor$ and $\left\lceil {\cdot} \right\rceil$ represent floor and ceiling operations, respectively.
\end{proposition}

\begin{proof}
We first conduct a functional analysis for the objective function in problem~\eqref{sub-problem_1}. Let $\Theta' \left( {\boldsymbol{b}} \right) = \frac{{2\vartheta \left( {\sum\limits_{i = 1}^N {{b_i}{C_i}}  + D} \right)}}{{\gamma \left( {A - \sum\limits_{i = 1}^N  \frac{B}{{{b_i}}}} \right)}}$, where $A = \varepsilon  - {\mathbbm{1}_{\{ I > 1\} }}4{\beta ^2}{\gamma ^2}{I^2}{T_1}$, $B = \frac{{\beta \gamma }}{{{N^2}}}\sum\limits_{j = 1}^L {\sigma _j^2}$, ${C_i} = \frac{{\sum\limits_{j = 1}^L {{\mu _{i,j}}} \left( {{\rho _L} - {\rho _j} + {\varpi _L} - {\varpi _j}} \right)}}{{{f_s}}}$, and $D = {T_3} + {T_4} + \frac{{{T_5} + {T_6}}}{I}$. Without loss of generality, we take the first-order derivative for the BS of the arbitrary $i'$-th edge device $b_{i'}$, yielding 
\begin{equation}\label{accuracy_cons}
\frac{{\partial \Theta' \left( {\boldsymbol{b}} \right)}}{{\partial b_{i'}}} = \frac{{2\vartheta }}{\gamma }\frac{\Xi \big( {\boldsymbol{b}} \big)}{{{\left( {A - \sum\limits_{i = 1}^N  \frac{B}{{{b_i}}}} \right)}^2}},
\end{equation}
where
\begin{equation}\label{E_I}
\Xi \big( {\boldsymbol{b}} \big) = {{C_{i'}}\bigg( {A - \sum\limits_{i = 1}^N  \frac{B}{{{b_i}}}} \bigg) - \bigg( {\sum\limits_{i = 1}^N {{b_i}{C_i}}  + D} \bigg)\frac{B}{{b_{i'}^2}}}.
\end{equation}

Since $\frac{{\partial \Xi \left( {\boldsymbol{b}} \right)}}{{\partial b_{i'}}} = \left( {\sum\limits_{i = 1}^N {{b_i}{C_i}}  + D} \right)\frac{{2B}}{{b_{i'}^3}}>0$, $\Xi \left( {\boldsymbol{b}} \right)$ is an increasing function of $b_{i'}$. Considering that $\mathop {\lim }\limits_{b_{i'} \to  + \infty } \Xi \left( \boldsymbol{b} \right) = \frac{{2\vartheta }}{\gamma }\frac{{{C_{i'}}}}{{\left( {A - \!\!\sum\limits_{k = 1(k \ne {i'})}^N {\frac{B}{{{b_k}}}} } \right)}}  > 0$ and $\mathop {\lim }\limits_{b_{i'} \to  0^+ } \Xi \left( \boldsymbol{b} \right) = - \frac{{2\vartheta }}{{\gamma B}}\left( {\sum\limits_{k = 1(k \ne {i'})}^N {{b_k}{C_k}}  + D} \right) < 0$, there must exist ${\hat b}_{i'} \in \left( {0, + \infty } \right)$ satisfying $\Xi \left( {\boldsymbol{b}} \right)|_{b_{i'}={\hat b}_{i'}} = 0$, indicating that the objective function first decreases and then increases with respect to $b_{i'}$ and . Therefore, the minimum is taken at $b_{i'}={\hat b}_{i'}$. Considering $N$ edge devices with diverse BSs, we generalize~\eqref{E_I} into a system of $N$ equations, i.e., $\frac{{\partial \Theta '({\boldsymbol{b}})}}{{\partial {\boldsymbol{b}}}} = [\frac{{\partial \Theta '({\boldsymbol{b}})}}{{\partial {b_{1}}}},\frac{{\partial \Theta '({\boldsymbol{b}})}}{{\partial {b_{2}}}},...,\frac{{\partial \Theta '({\boldsymbol{b}})}}{{\partial {b_N}}}] = {\bf{0}}$. To solve this system of equations, we can utilize the Newton-Jacobi method~\cite{geng2019pipeline}, yielding the solution ${{\hat{\boldsymbol{b}}}} = [ {{{\hat b}_{1}},{{\hat b}_{2}},...,{\hat b}_N}]$. 
According to constraint $\mathrm{C5}$, which requires the BS to be a positive integer, the optimal batch size $b^*_{i'}$ for the $i'$-th edge device must be one of the two integers closest to the continuous solution $\hat{b}_{i'}$. Combining this insight with constraints $\mathrm{C4}$, $\mathrm{R3}$ and $\mathrm{R4}$, the optimal solution is given by

{\begin{align}\label{accuracy_cons}
{{\bf{b}}^*} = \left[ {b_1^*,b_2^*,...,b_N^*} \right],
\end{align}}
where $b_i^*$ is determined as
{\begin{align}\label{accuracy_cons}
b_i^* = \left\{ {\begin{array}{*{20}{c}}
1&{{{\hat b}_i} \le 1}\\
{\arg {{\min }_{{b_i} \in \left\{ {\left\lfloor {{{\hat b}_i}} \right\rfloor ,\left\lceil {{{\hat b}_i}} \right\rceil } \right\}}}\Theta '\left( {\boldsymbol{b}} \right)}&{1 < {{\hat b}_i} < {\kappa _i}}\\
{\left\lfloor {{\kappa _i}} \right\rfloor}&{{{\hat b}_i} \ge {\kappa _i}},
\end{array}} \right.
\end{align}}
Here, \resizebox{0.90\hsize}{!}{$
\kappa_i \! =\! \min \!\left\{\!\!
\frac{ \upsilon_{c,i} - \sum\limits_{j=1}^{L} \mu_{i,j} ( \vartheta_j + \delta_j ) }
     { \sum\limits_{j=1}^{L} \mu_{i,j} ( \psi_j + \chi_j ) },\;\!\!
\frac{ T_3 }
     { \sum\limits_{j=1}^{L} \mu_{i,j} \left( \frac{ \rho_j }{ f_i } + \frac{ \psi_j }{ r_i^U } \right) },\;\!
\frac{ T_4 }
     { \sum\limits_{j=1}^{L} \mu_{i,j} \left( \frac{ \chi_j }{ r_i^D } + \frac{ \varpi_j }{ f_i } \right) }
\!\!\right\}
$}.
\end{proof}

{\textbf{Remark:}} 
Due to constraints $\mathrm{C4}$, $\mathrm{R3}$ and $\mathrm{R4}$, the BS for each edge device has three potential solutions, i.e., $1$, $\lfloor \hat{b}_i \rfloor$ or $\lceil \hat{b}_i \rceil$, and $\lfloor \kappa_i \rfloor$. Based on {\bf Proposition 1}, the optimal solution to problem~\eqref{sub-problem_1} can be obtained via exhaustive search over all $3^N$ possible combinations -- a reduction from the original $B^N$ combinations, where $B$ is the maximum BS. The objective function is calculated for each combination, and the one with the minimum value is identified as the global optimum. While exhaustive search is feasible for small-scale systems, it becomes computationally prohibitive as $N$ increases. To overcome this scalability challenge, an efficient sub-optimal solution can be derived using \textbf{Proposition 1}. Specifically, by solving the first-order condition $\frac{\partial \Theta'({\boldsymbol{b}})}{\partial {\boldsymbol{b}}} = {\bf 0}$, $\hat{{\boldsymbol{b}}} = \{ \hat{b}_i \mid i \in N \}$ is obtained. Each $\hat{b}_i$ is then discretized and adjusted using Eqn.~\eqref{b_three_solution} to yield the final integer BS $b_i^*$. This approach only requires solving a nonlinear system of $N$ variables followed by a one-time correction step, significantly reducing computational complexity.

Therefore, we obtain an efficient solution to problem~\eqref{sub-problem_1}.

\begin{algorithm}[t]
	\renewcommand{\algorithmicrequire}{\textbf{Input:}}
	\renewcommand{\algorithmicensure}{\textbf{Output:}}
	\caption{BCD-based Algorithm.}\label{BCD-based}
	\begin{algorithmic}[1]
        \REQUIRE  ${\boldsymbol{b}}^{(0)}$, ${\boldsymbol{\mu }}^{(0)}$, ${{\bf{T}}}^{(0)}$ and ${\varepsilon}$.
		\ENSURE  ${\boldsymbol{b}^{\bf{*}}}$ and ${{\boldsymbol{\mu }}^{\bf{*}}}$  
          \STATE Initialization:  $\tau  \leftarrow 0$.
          \REPEAT 
          \STATE$\tau  \leftarrow \tau+1$
           \STATE Update ${\boldsymbol{b}}^{(\tau)}$ by solving problem~\eqref{sub-problem_1}
           \STATE Update ${\boldsymbol{\mu }}^{(\tau)}$ and  ${{\bf{T}}}^{(\tau)}$ by solving problem~\eqref{sub-problem_2}
         \UNTIL{\small{$|\Theta' ( {\boldsymbol{b}}^{(\tau)}, {{\boldsymbol{\mu }}^{(\tau)}}, {{\bf{T}}^{(\tau)}} ) - \Theta' ( {\boldsymbol{b}}^{(\tau-1)}, {\boldsymbol{\mu }}^{(\tau-1)}, {{\bf{T}}^{(\tau-1)}}| \!\le \!{\varepsilon}$}}    
         
	\end{algorithmic}  
\end{algorithm}


{\textit{MS sub-problem.}} By fixing the decision variable $\boldsymbol{b}$, we transform problem~\eqref{problem_2} into a standard mixed-integer linear fractional  programming with respect to $\boldsymbol{\mu}$ and ${\bf{T}}$, which is expressed as
\begin{align}\label{sub-problem_2}
\mathcal{P}_2:&\mathop {{\rm{min}}}\limits_{{\boldsymbol{\mu}, {\bf{T}}}} \Theta' ( {\boldsymbol{\mu }}, {\bf{T}} )   \\
&\mathrm{s.t.} ~\mathrm{C2}-\mathrm{C4},~\mathrm{R1}-\mathrm{R8}. \nonumber
\end{align}

We leverage the Dinkelbach algorithm~\cite{dinkelbach1967nonlinear} to solve problem~\eqref{sub-problem_2} by transforming it into mixed-integer linear programming, which guarantees the optimal solution~\cite{yue2013reformulation,rodenas1999extensions}.

As mentioned earlier, we decompose the original problem~\eqref{time_minimize_problem} into two tractable sub-problems $\mathcal{P}1$ and $\mathcal{P}2$ according to decision variables and develop efficient algorithms to solve
each sub-problem. We design an iterative block-coordinate descent (BCD)-based algorithm~\cite{tseng2001convergence} to solve problem~\eqref{time_minimize_problem}, as described in~\textbf{Algorithm~\ref{BCD-based}}. The key parameters required for executing the algorithm (e.g., $\beta$, ${G _j^2}$ and ${\sigma _j^2}$) are estimated following the approach in~\cite{wang2019adaptive}.

\section{Performance Evaluation}\label{simu_results}
This section provides numerical results to evaluate the performance of HASFL from three aspects: i) comparisons with four benchmarks to demonstrate the superiority of HASFL; ii) investigating the robustness
of HASFL to varying network computing and communication
resources, and the number of edge devices; iii) conducting the ablation study to validate the necessity of each meticulously designed component in HASFL, including BS and MS.

\begin{table}[t]\label{table_2}
  \centering
  \caption{Simulation Parameters.}
  \renewcommand{\arraystretch}{1.3}{
  \setlength{\tabcolsep}{1.3mm}{
\begin{tabular}{|c|c|c|c|}
\hline
\textbf{Parameter}          & \textbf{Value} & \textbf{Parameter} & \textbf{Value}  \\ \hline
$f_s$             & 20 TFLOPS             & $f_i$                 & [1, 2] TFLOPS                          \\ \hline
$N$             & 20              & ${r_i^{U}}$/$r_{i,f}^U$           & [75, 80] Mbps                   \\ \hline
${r_i^{D}}$/$r_{i,f}^D$               & [360, 380] Mbps              & ${r_s^{U}}$/${r_s^{D}}$                  & [360, 380] Mbps                      \\ \hline
$\gamma$             & $5\times {10^{-4}}$           &      $I$       &    15                  \\ \hline

\end{tabular}}}
\end{table}

\subsection{Experimental Setup}\label{simu_setup}

{\bf{Implementation and hyper-parameters.}}  HASFL is implemented using Python 3.7 and PyTorch 1.9.1. All training procedures are conducted on a ThinkPad P17 Gen1 workstation, which is configured with an NVIDIA Quadro RTX 3000 GPU, an Intel i9-10885H processor, and a 4TB solid-state drive. We deploy $N$ edge devices, and $N$ is set to 20 by default unless specified otherwise. The computing capability of each edge device is uniformly distributed within [1, 2] TFLOPS, and the edge server is provisioned with a computing capability of 20 TFLOPS. The uplink data rates from the $i$-th edge device to the edge server and fed server follow uniform distribution within [75, 80] Mbps, and the corresponding downlink data rates are uniformly distributed within [360, 380] Mbps. For convenience, the inter-server data rate, namely, ${r_s^{U}}$ and ${r_s^{D}}$, also identically follow uniform distribution within [360, 380] Mbps. The client-side sub-model aggregation interval and learning rate are set to 15 and $5\times{10^{-4}}$, respectively. For readers' convenience, the detailed experiment parameters are summarized in Table I.

\begin{figure}[t]
\setlength\abovecaptionskip{3pt}
\centering
\subfigure[CIFAR-10 on VGG-16 under IID setting.]{
\includegraphics[width=0.45\linewidth]{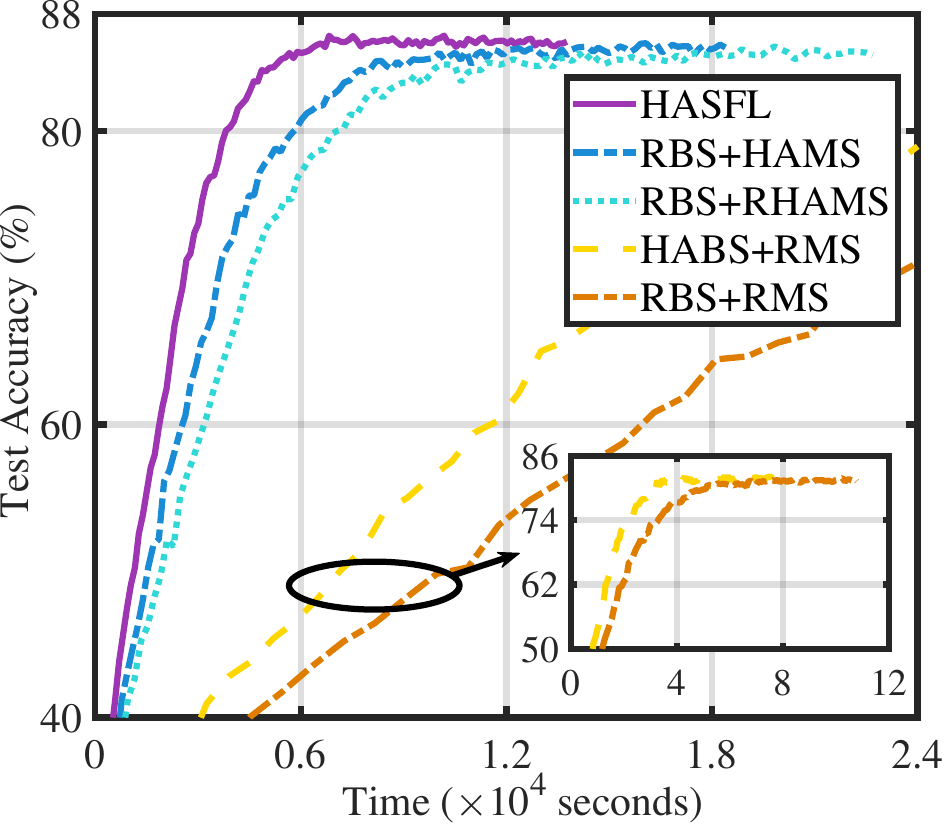}
    \label{sfig:cifar_iid_test_accuracy}
}
\subfigure[CIFAR-10 on VGG-16 under non-IID setting.]{
    \includegraphics[width=0.448\linewidth]{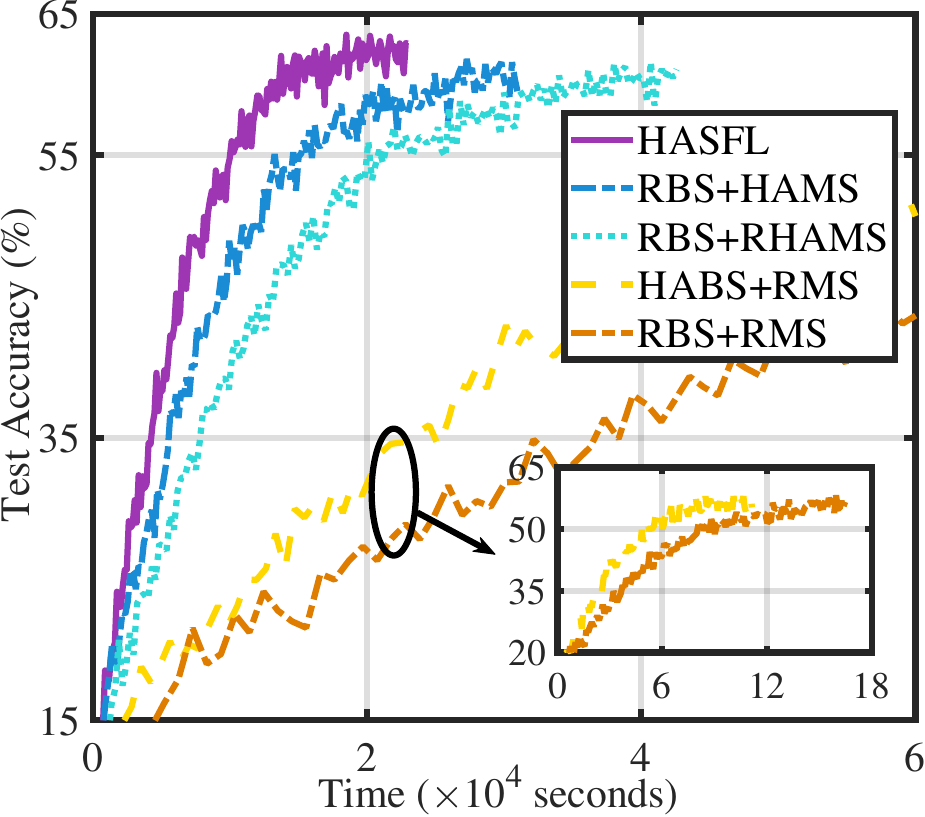}
    \label{sfig:cifar_non_iid_test_accuracy}
}
\subfigure[ CIFAR-100 on ResNet-18 under IID setting.]{
\includegraphics[width=0.448\linewidth]{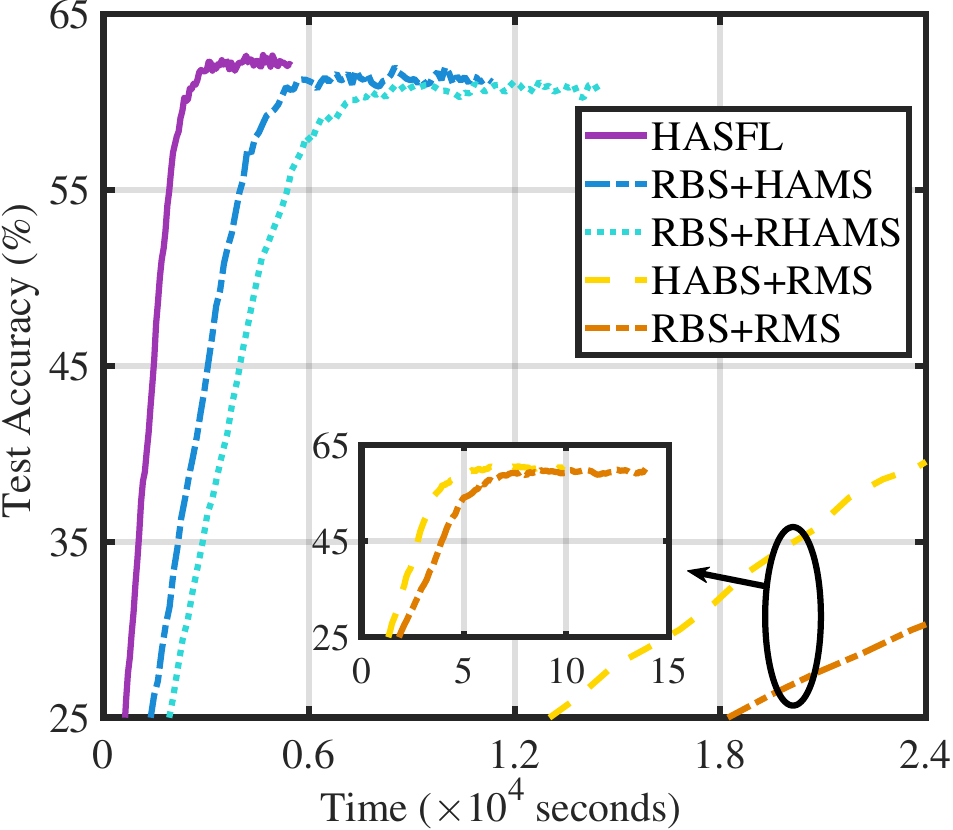}
    \label{sfig:mnist_iid_test_accuracy}}
\subfigure[CIFAR-100 on ResNet-18 under non-IID setting.]{
    \includegraphics[width=0.45\linewidth]{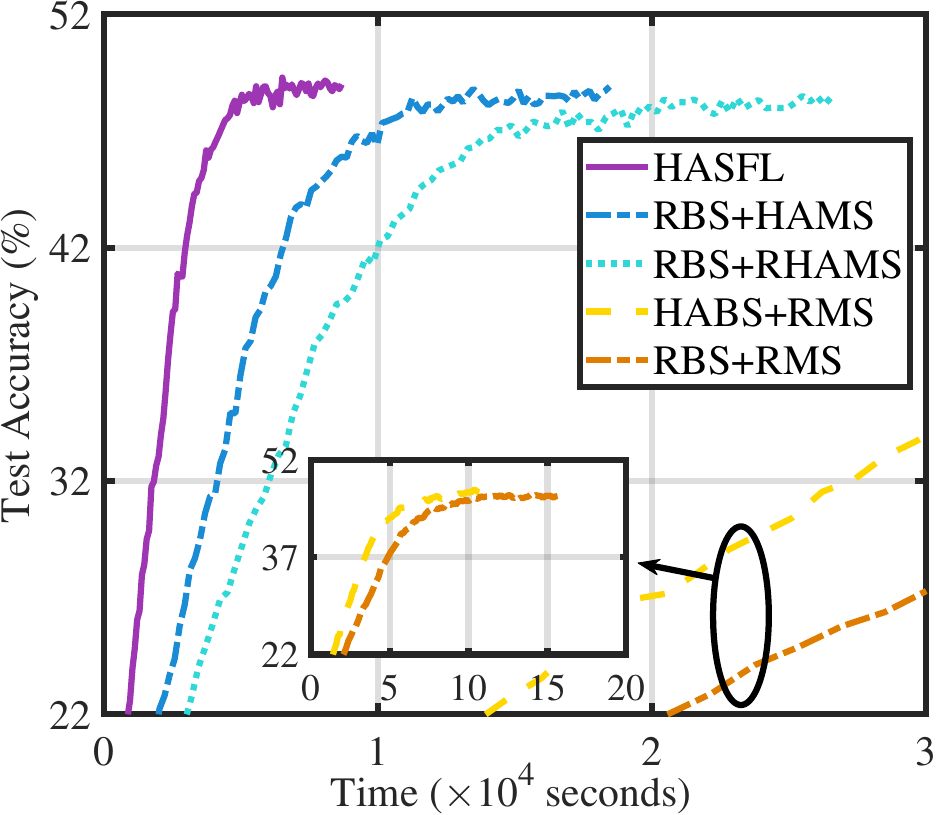}
    \label{sfig:mnist_non_iid_test_accuracy}
}
    \caption{The training performance for CIFAR-10 and CIFAR-100 datasets under IID and non-IID settings using VGG-16 and ResNet-18.}
    \label{fig:test_accuracy}
\end{figure}

{\bf{Dataset and model.}} To evaluate the learning performance of HASFL, we adopt two widely used image classification datasets, CIFAR-10 and CIFAR-100~\cite{krizhevsky2009learning}. The CIFAR-10 dataset consists of 10 distinct categories of object images, such as airplanes and trucks, and contains 50000 training samples and 10000 test samples. The CIFAR-100 dataset comprises object images from 100 categories, with each category containing 500 training samples and 100 test samples. The experiments are conducted under IID and non-IID data settings. In the IID setting, the dataset is randomly shuffled and evenly allocated across all edge devices.  In the non-IID setting~\cite{zhu2019broadband,yang2020energy,lin2024fedsn}, the dataset is first sorted based on class labels, and then partitioned into 40 shards, with each of 20 edge devices receiving two randomly distributed shards. Additionally, we employ the well-known ML models, VGG-16~\cite{simonyan2014very} and ResNet-18~\cite{he2016deep} for performance evaluation. VGG-16 is a classical deep convolutional neural network composed of 13 convolution layers and 3 fully connected layers, while ResNet-18 is a residual neural network consisting of 17 convolutional layers and 1 fully connected layer.

{\bf{Benchmarks.}} To comprehensively evaluate the performance of HASFL, we compare it against the following alternatives:
\begin{itemize}
    \item {\bf{RBS+HAMS:}} This benchmark utilizes a random BS strategy (i.e., randomly drawing BS from 1 to 64 during model training), and employs the tailored heterogeneity-aware MS scheme in Section~\ref{solu_appro}.
    \item {\bf{HABS+RMS:}} This benchmark adopts the heterogeneity-aware BS strategy in Section~\ref{solu_appro} and deploys a random MS scheme (i.e., randomly selecting model split points during model training).
    \item {\bf{RBS+RMS:}} This benchmark employs the random BS and MS strategy. 
    \item {{\bf{RBS+RHAMS:}} This benchmark utilizes the RBS scheme and adopts a resource-heterogeneity-aware MS strategy~\cite{wang2023coopfl}.}
\end{itemize}


\subsection{Superiority of HASFL}

\begin{figure}[t]
\setlength\abovecaptionskip{3pt}
\centering
\subfigure[CIFAR-10 on VGG-16 under IID setting.]{
\includegraphics[width=0.448\linewidth]{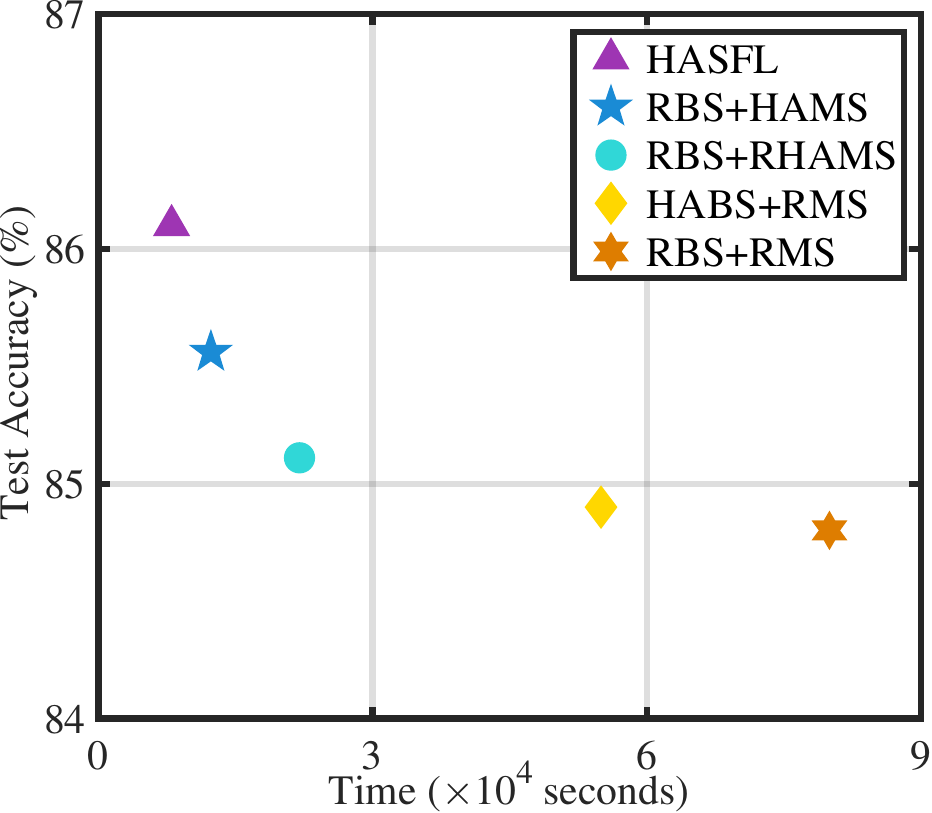}
    \label{sfig:cifar_iid_time_accuracy}
}
\subfigure[CIFAR-10 on VGG-16 under non-IID setting.]{
    \includegraphics[width=0.448\linewidth]{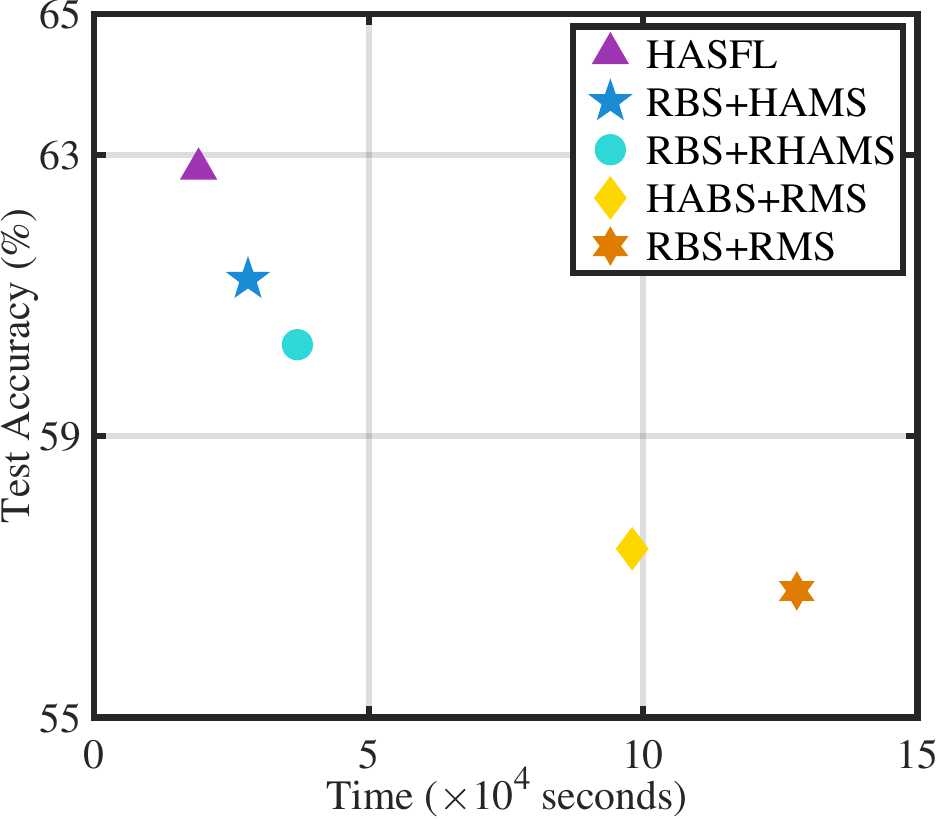}
    \label{sfig:cifar_non_iid_time_accuracy}
}
\subfigure[CIFAR-100 on ResNet-18 under IID setting.]{
\includegraphics[width=0.448\linewidth]{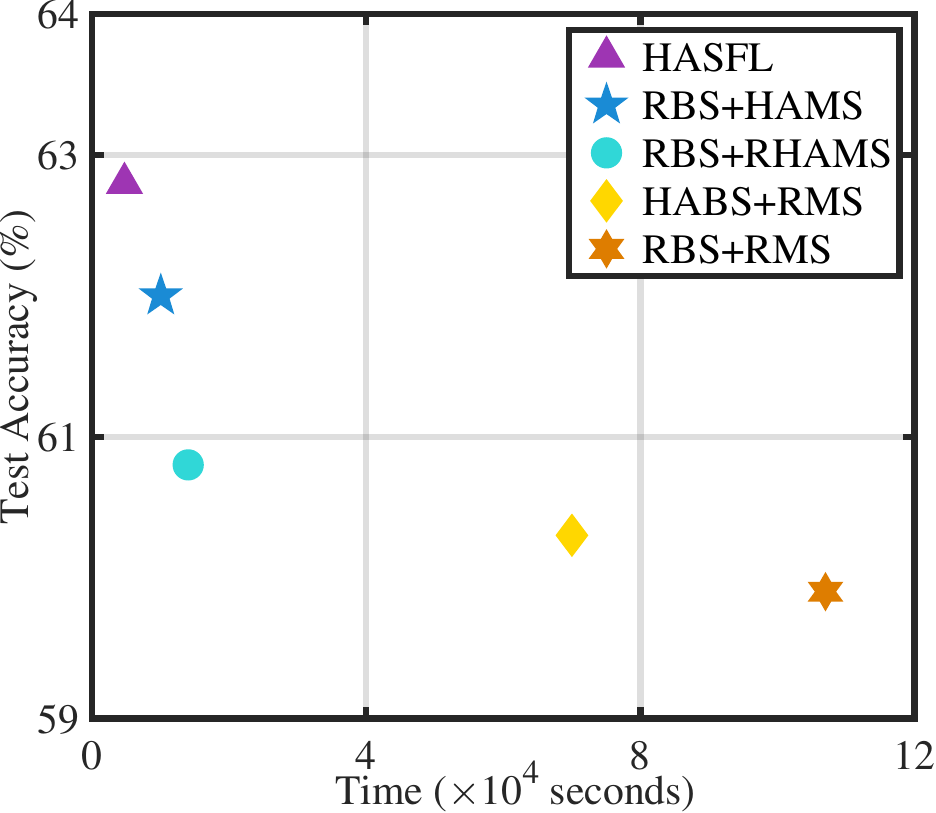}
    \label{sfig:mnist_iid_time_accuracy}
}
\subfigure[CIFAR-100 on ResNet-18 under non-IID setting.]{
    \includegraphics[width=0.448\linewidth]{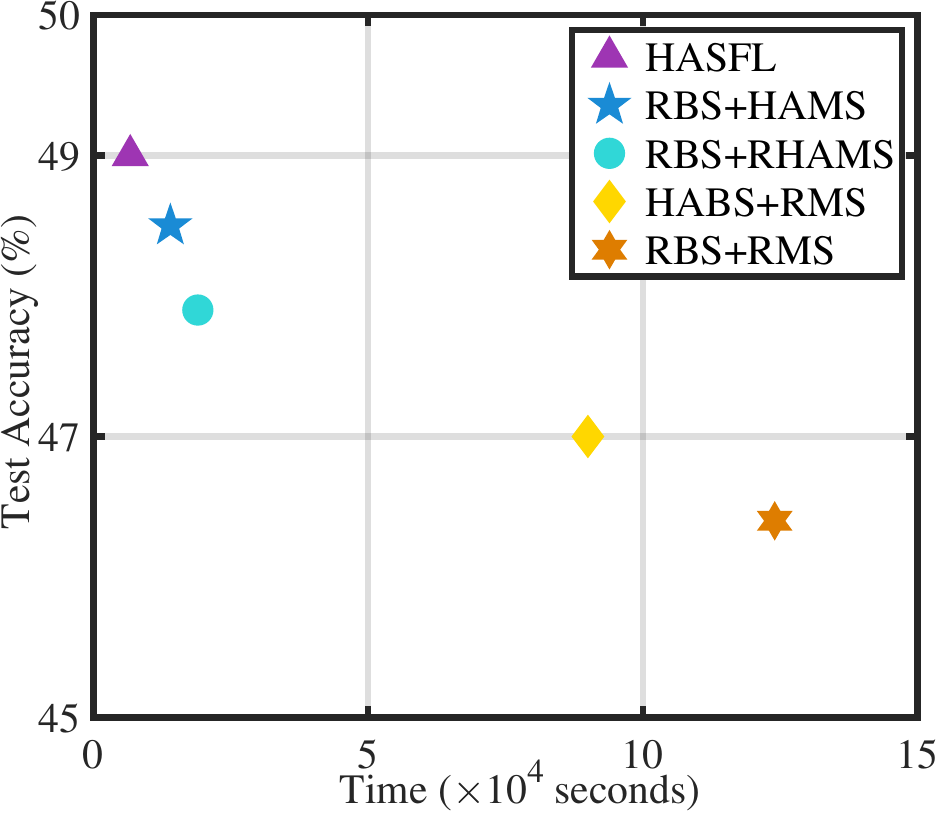}
    \label{sfig:mnist_non_iid_time_accuracy}
}
    \caption{The converged accuracy and time for CIFAR-10 and CIFAR-100 datasets under IID and non-IID settings using VGG-16 and ResNet-18.}
    \label{fig:time_accuracy}
\end{figure}

Fig.~\ref{fig:test_accuracy} presents the training performance of HASFL and four benchmarks on the CIFAR-10 and CIFAR-100 datasets. HASFL demonstrates superior convergence speed and improved test accuracy compared to other benchmarks. Notably, HASFL, RBS+HAMS, and RBS+RHAMS achieve significantly faster model convergence than HABS+RMS and RBS+RMS, primarily due to heterogeneity-aware MS, which strikes a balance between communication-computing overhead and training convergence. Furthermore, the performance comparison between HASFL and RBS+HAMS reveals the advantage of the tailored BS scheme, which can expedite the model training while retaining training accuracy. The performance discrepancy between RBS+HAMS and RBS+RHAMS is mainly attributed to the heuristic design of RBS+RHAMS, which lacks systematic optimization capturing the interplay between model convergence and MS. By comparing Fig.~\ref{sfig:cifar_iid_test_accuracy} with Fig.~\ref{sfig:cifar_non_iid_test_accuracy}, and Fig.~\ref{sfig:mnist_iid_test_accuracy} with Fig.~\ref{sfig:mnist_non_iid_test_accuracy}, we show that HASFL and the other four benchmarks consistently converge more slowly under the non-IID setting than IID setting.

\begin{figure}[t!]
\setlength\abovecaptionskip{3pt}
\centering
\subfigure[Computing capabilities of edge devices.]{
\includegraphics[width=0.448\linewidth]{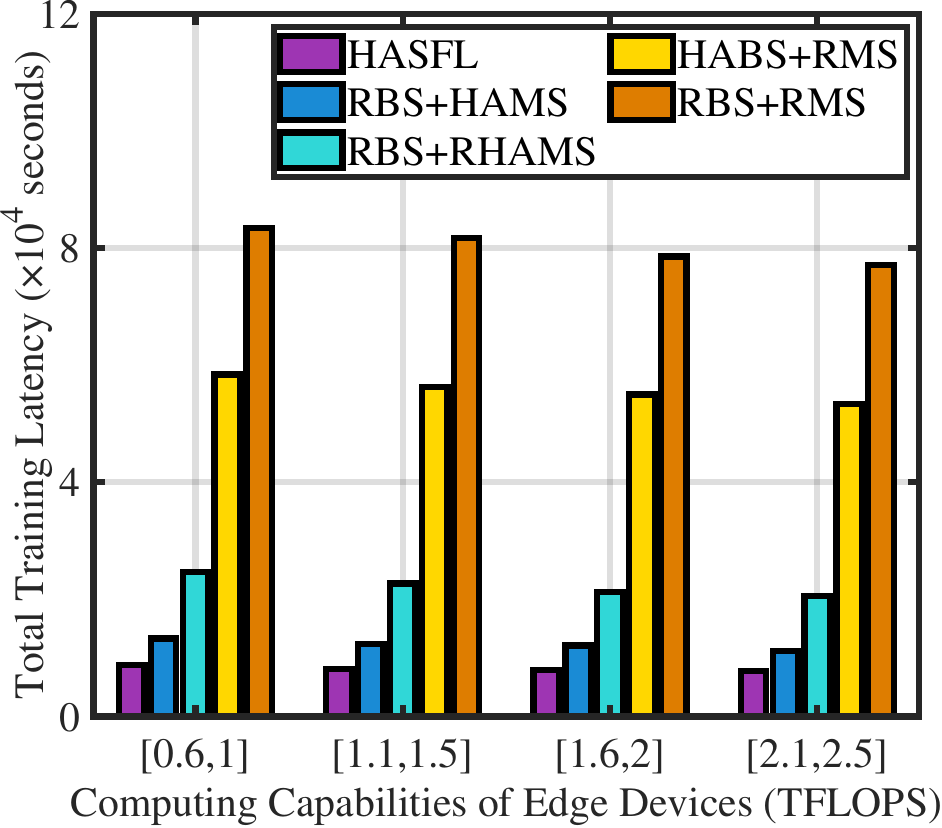}
    \label{sfig:device_comput_accuracy}
}
\subfigure[Computing capabilities of edge server.]{
    \includegraphics[width=0.448\linewidth]{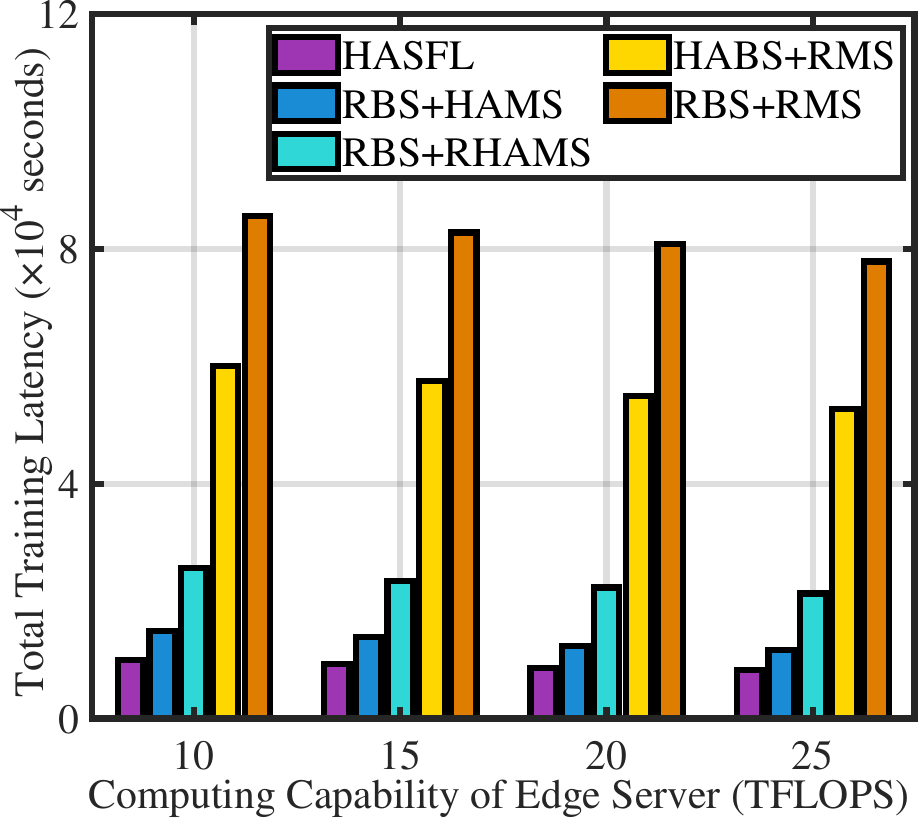}
    \label{sfig:server_comput_accuracy}
}
    \caption{ The converged time versus network computing resources on CIFAR-10 under IID setting using VGG-16.}
    \label{fig:comput_accuracy}
\end{figure}

\begin{figure}[t]
\setlength\abovecaptionskip{3pt}
\centering
\subfigure[Uplink rates of edge devices.]{
\includegraphics[width=0.448\linewidth]{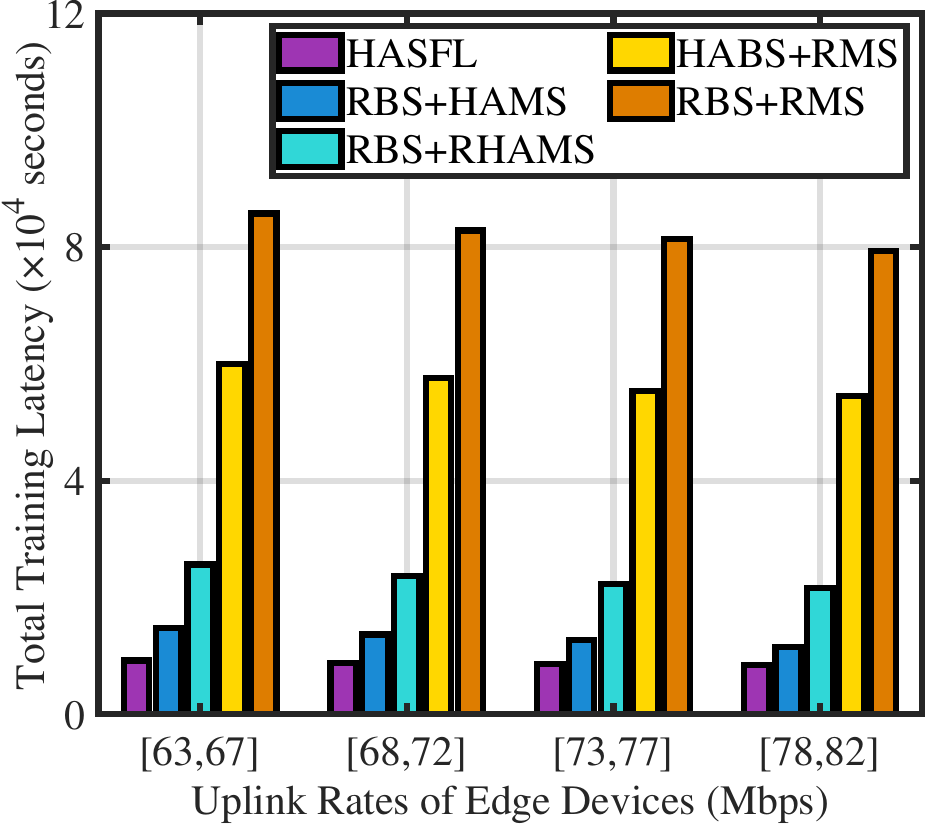}
    \label{sfig:device_commu_accuracy}
}
\subfigure[Inter-server communication rate.]{
    \includegraphics[width=0.448\linewidth]{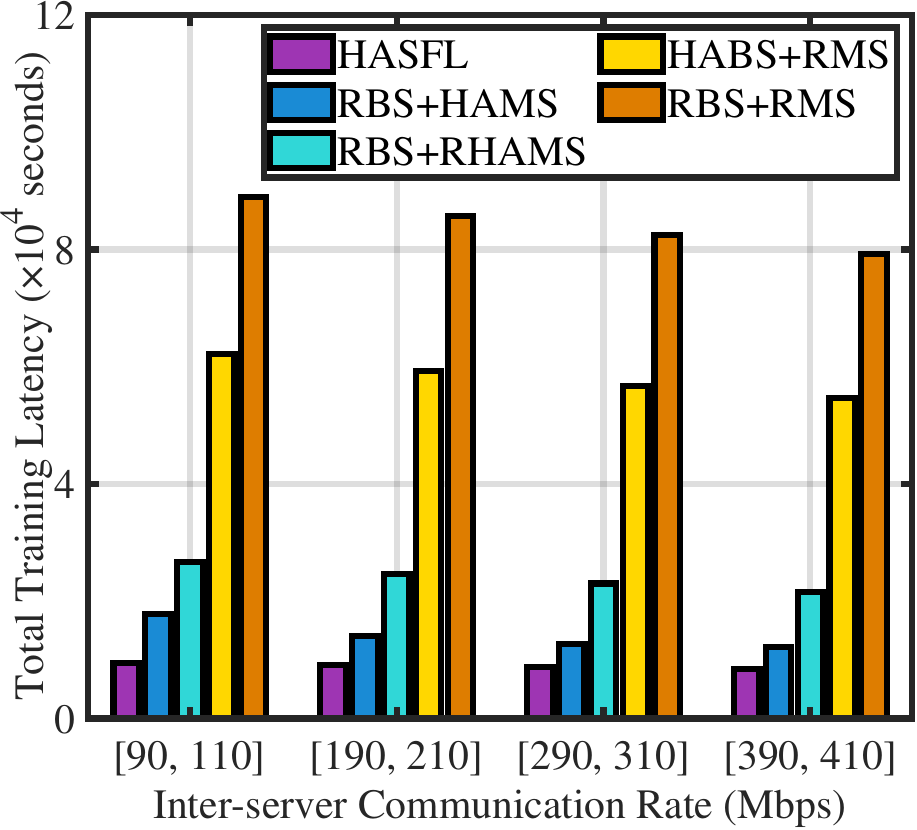}
    \label{sfig:server_commu_accuracy}
}
    \caption{  The converged time versus network communication resources on CIFAR-10 under IID setting using VGG-16.}
    \label{fig:commu_accuracy}
\end{figure}

Fig.~\ref{fig:time_accuracy} illustrates the converged accuracy and time (i.e., the test accuracy increases by less than 0.02\% across five consecutive training rounds) on the CIFAR-10 and CIFAR-100 datasets. Comparing HASFL with RBS+HAMS and HABS+RMS reveals a significant impact of BS and MS on both converged accuracy and time. Notably, MS plays a more decisive role in model training than BS. This is because the selected model split point directly affects the overall model aggregation frequency. For instance, a shallower model split point (i.e., fewer layers placed on the edge device) demonstrates that larger portions of the global model are aggregated at each training round. In the IID setting, the non-optimized MS scheme leads to an accuracy drop by 1.3$\%$ (resp. 2.4$\%$) and nearly a sixfold (resp. fifteenfold) increase in the converged time on CIFAR-10 (resp. CIFAR-100) dataset, which is 0.7$\%$ and 4.4 times (resp. 1.6$\%$ and 7 times) that of the non-optimized MS strategy. In the non-IID setting, HASFL still significantly outperforms the non-optimized MS scheme in both converged accuracy and time, achieving approximately 5.6$\%$ (resp. 2$\%$) accuracy increase and  4.8 (resp. 13.4) times faster model convergence acceleration on the CIFAR-10 (resp. CIFAR-100) dataset.  The reason behind this is the model bias introduced by the heterogeneous nature of local data distributions, consistent with FL. By comparing HASFL with RBS+RMS, HASFL achieves a faster model convergence of 9.4 (resp. 6.4) times over its counterpart without optimization on the CIFAR-10 (resp. CIFAR-100) dataset while guaranteeing training accuracy, showcasing the superior performance of HASFL. Though RBS+HAMS, HABS+RMS, and RBS+RHAMS exhibit better training performance over RBS+RMS, we observe that the performance gains remain relatively limited because they lack a unified optimization of BS and MS.

\subsection{Impact of Varying Network Resources}

Figs.~\ref{fig:comput_accuracy} and~\ref{fig:commu_accuracy} show the converged time versus network computing and communication resources on CIFAR-10 under the IID setting.  HASFL consistently achieves faster model convergence than the other benchmarks across varying network resources. The training performance of RBS+RMS deteriorates significantly as network resources decline. This degradation arises from the lack of optimization of BS and MS, leading to sub-optimal trade-offs between communication-computing overhead and model convergence. The comparison of HASFL with RBS+HAMS and HABS+RMS reveals that optimizing either BS or MS can compensate for the performance decline caused by the reduced network resources. Though RBS+RHAMS achieves faster model convergence than RBS+RMS, its performance gains remain limited by the lack of joint optimization of BS and MS. In contrast, HASFL demonstrates better robustness over other benchmarks, with the converged time of HASFL exhibiting only a marginal increase as network resources diminish, owing to its heterogeneity-aware BS and MS design. Specifically, HASFL dynamically adapts BS and MS based on network resource conditions to achieve the minimum rounds required for model convergence to expedite model training, demonstrating the superior adaptability to varying network resources.

\begin{figure}[t]
\setlength\abovecaptionskip{3pt}
\centering
\subfigure[{CIFAR-10 under IID setting.}]{
\includegraphics[width=0.448\linewidth]{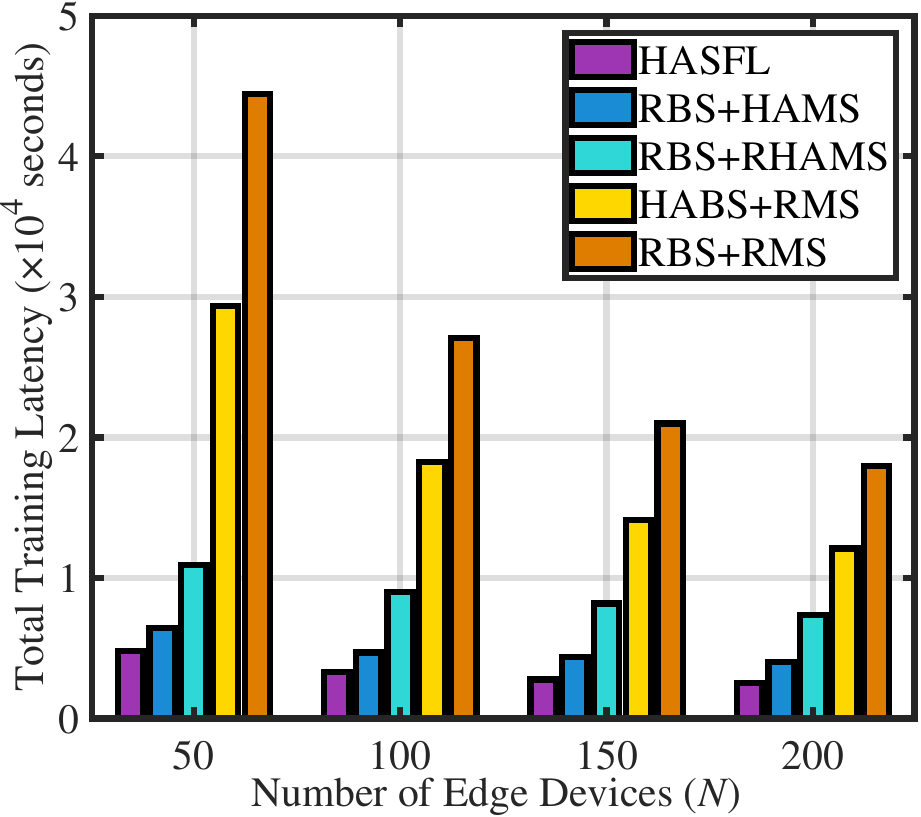}
    \label{sfig:cifar_iid_num}
}
\subfigure[{CIFAR-10 under non-IID setting.}]{
    \includegraphics[width=0.448\linewidth]{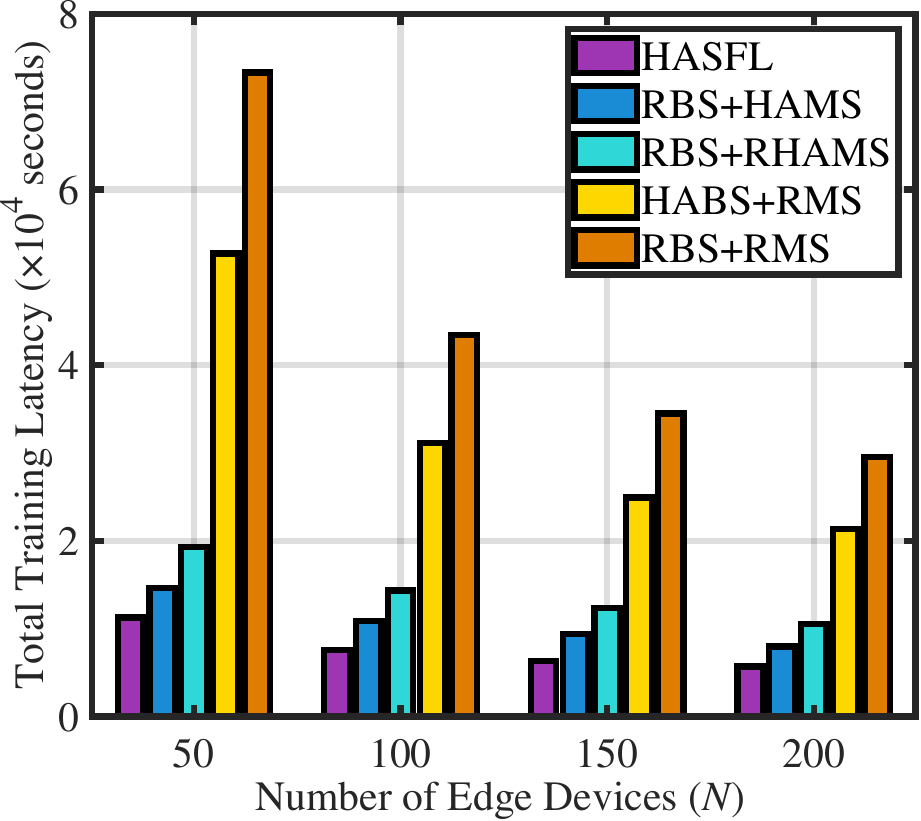}
    \label{sfig:cifar_non_iid_num}
}
    \caption{{The converged time versus number of edge devices on CIFAR-10 under IID and non-IID setting using VGG-16.}}
    \label{fig:cifar_deivce_num}
\end{figure}

\subsection{Impact of Number of Edge Devices}


{Fig.~\ref{fig:cifar_deivce_num} presents the converged time versus the number of edge devices on CIFAR-10 under the IID and non-IID settings. HASFL consistently outperforms the other benchmarks by achieving significantly lower training latency across all configurations of $N$. Notably, the performance gap between HASFL and the other benchmarks becomes increasingly pronounced with the increasing number of edge devices, underscoring HASFL’s superior scalability and efficiency in large-scale edge networks. The converged time of HASFL and all benchmarks is noticeably longer 
under the non-IID setting than under IID setting. This degradation in performance is primarily attributed to the statistical heterogeneity of non-IID data, which often results in inconsistent local updates across edge devices and hinders model convergence. HASFL maintains its advantage over the other benchmarks, underscoring its robustness to diverse and imbalanced data distributions.
}

\begin{figure}[t]
\setlength\abovecaptionskip{3pt}
\centering
\subfigure[CIFAR-10 under IID setting.]{
\includegraphics[width=0.445\linewidth]{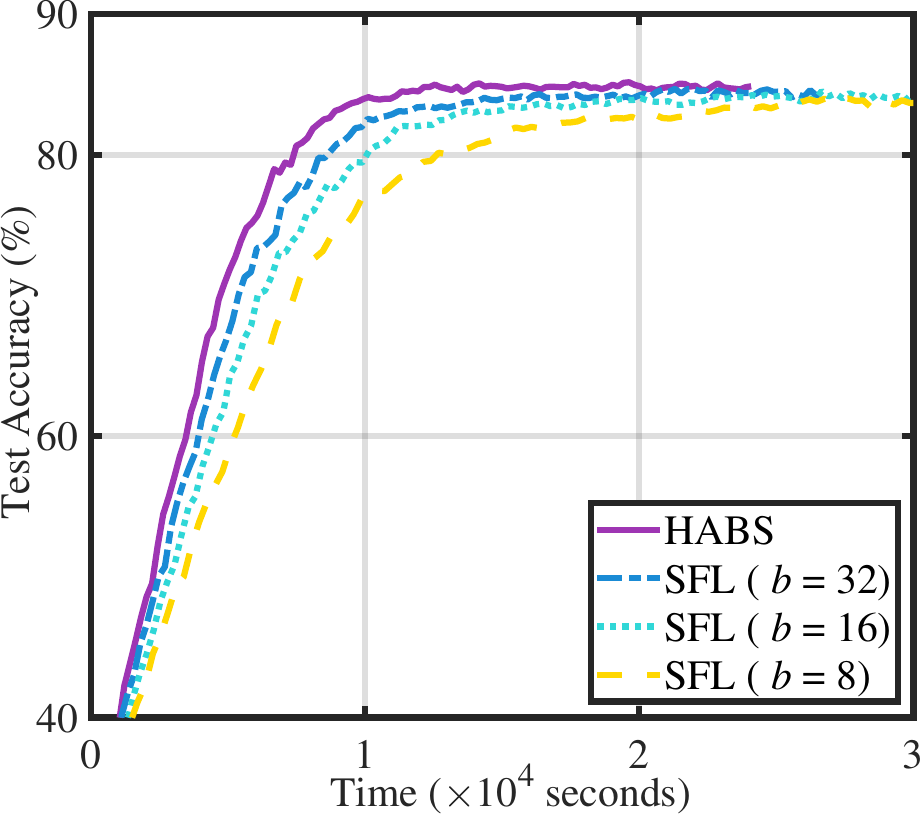}
    \label{sfig:cifar_iid_cut_aba}
}
\subfigure[CIFAR-10 under non-IID setting.]{
    \includegraphics[width=0.445\linewidth]{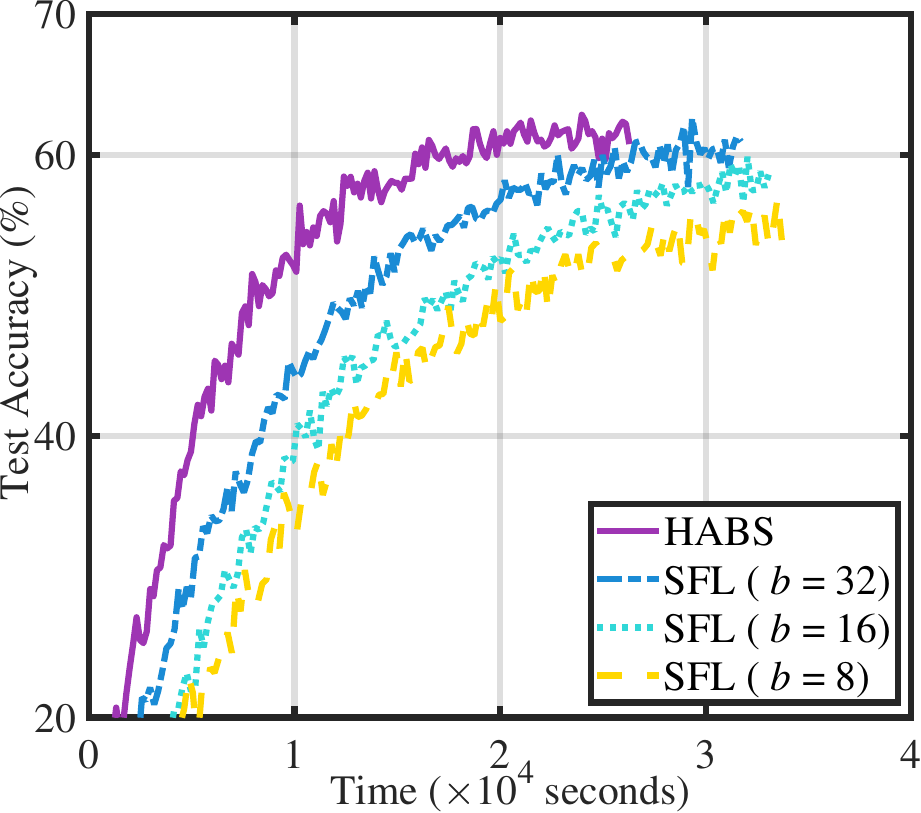}
    \label{sfig:cifar_non_iid_cut_aba}
}
    \caption{Ablation experiments for HABS strategy on the CIFAR-10 dataset under IID and non-IID setting using VGG-16 with $L_c=8$.}
    \label{fig:cifar_cut_aba}
\end{figure}

\begin{figure}[t]
\setlength\abovecaptionskip{3pt}
\centering
\subfigure[CIFAR-10 under IID setting.]{
\includegraphics[width=0.445\linewidth]{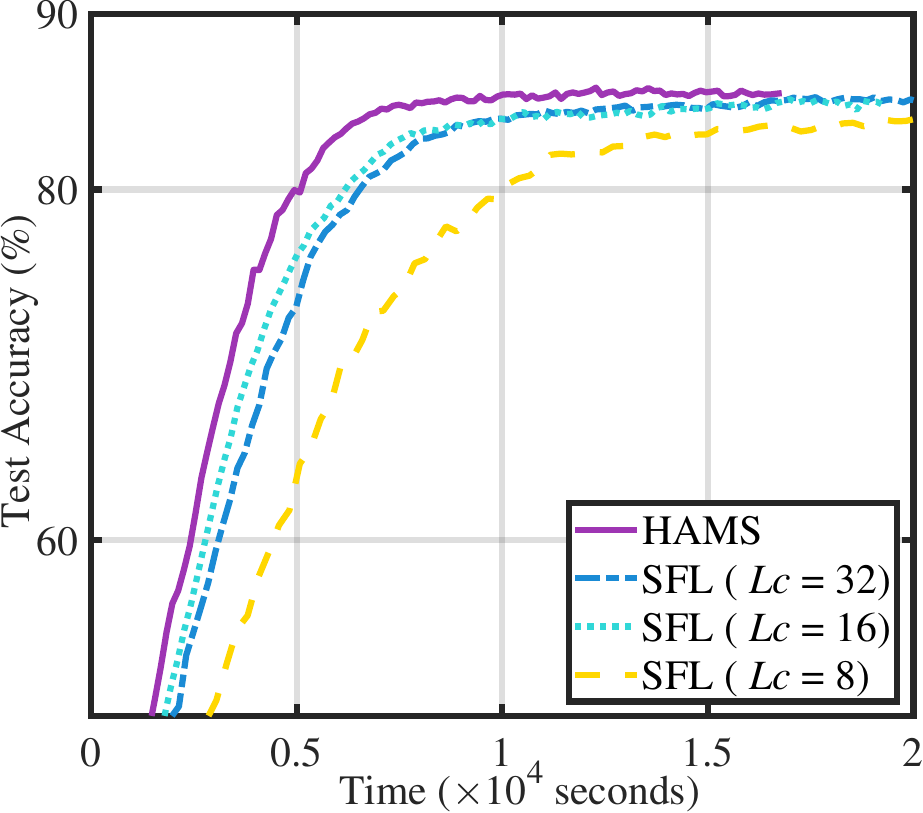}
    \label{sfig:cifar_iid_I_15_different_cut}
}
\subfigure[CIFAR-10 under non-IID setting.]{
    \includegraphics[width=0.454\linewidth]{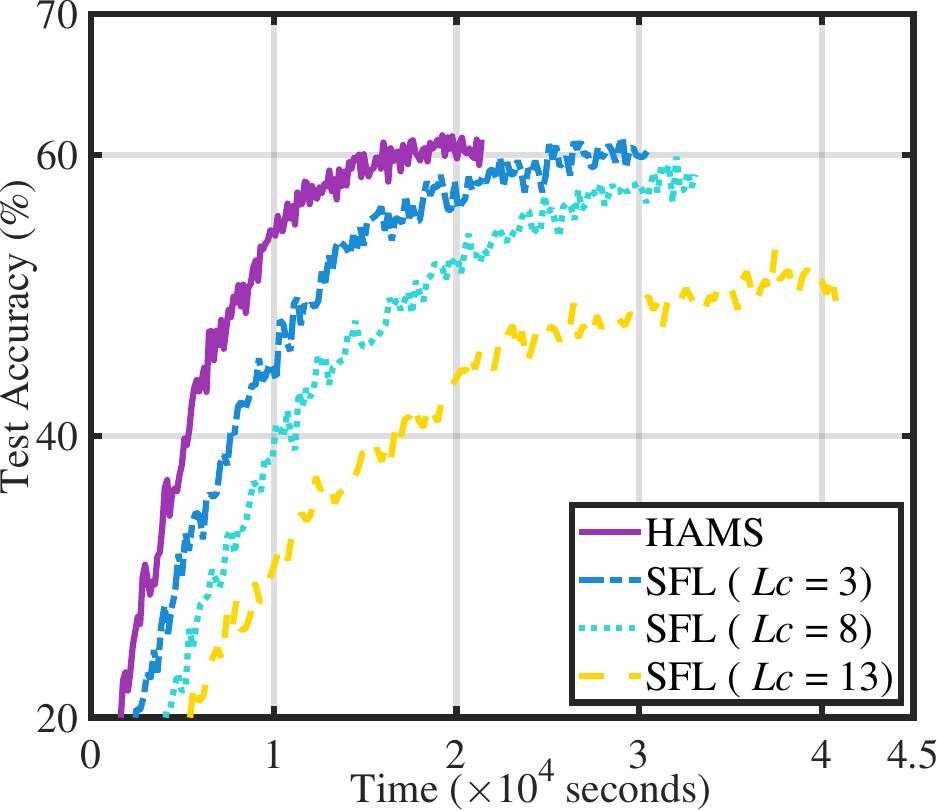}
    \label{sfig:cifar_non_iid_I_15_different_cut}
}
    \caption{Ablation experiments for HAMS scheme on the CIFAR-10 dataset under IID and non-IID setting using VGG-16 with $b=16$.}
    \label{fig:I_15_different_cut}
\end{figure}

\subsection{Ablation Study of The HASFL Framework}\label{simu_setup}

Fig.~\ref{fig:cifar_cut_aba} illustrates the impact of BS on training performance for the CIFAR-10 dataset. The proposed HABS scheme achieves faster model convergence and higher accuracy than fixed BS benchmarks, where all edge devices utilize the same BS ($b = $8, 16, and 32). Specifically, HABS achieves at least 0.9$\%$ and 1$\%$ accuracy improvements and reduces the convergence time by a factor of 1.4 and 1.6 compared to fixed BS benchmarks under the IID and non-IID settings, respectively. This performance gain primarily stems from the principled design of HABS, which leverages a theoretical convergence bound to determine BS. By dynamically assigning the optimal BS to each edge device based on network resource conditions and training convergence, HABS effectively mitigates the straggler effect to expedite model training under heterogeneous edge systems. Hence, the effectiveness of  HABS is validated.

Fig.~\ref{fig:I_15_different_cut} presents the impact of MS on training performance for the CIFAR-10 dataset. Both converged time and accuracy degrade as the model split point $L_c$ becomes deeper, consistent with the derived convergence bound in Eqn.~\eqref{convergence_bound}. This performance decline occurs because a deeper model split point places fewer layers on the server-side sub-model, reducing the proportion of the global model that is aggregated at each round. Consequently, the overall update frequency of the global model decreases, slowing model convergence and impairing the model's generalization. Furthermore, the impact of MS is more significant under the non-IID than IID setting, as the heterogeneity of local datasets exacerbates the sensitivity of MS under the non-IID setting. The comparisons between HAMS and the other three benchmarks reveal that HAMS can expedite model convergence and improve training performance. This demonstrates the effectiveness of the HAMS design in adapting to heterogeneous network resources.

\section{Conclusions}\label{conclu}

In this paper, we have proposed a heterogeneity-aware SFL framework, named HASFL, to minimize the training latency of SFL under heterogeneous resource-constrained edge computing systems. We have observed that both BS and MS significantly affect training latency, control of which can mitigate the straggler issue to a great extent. To guide the system optimization, we have first derived the convergence bound of HASFL to quantify the impact of BS and MS on training convergence. Then, we have formulated a joint optimization problem for BS and MS based on the derived convergence bound, and developed efficient algorithms to solve it. Extensive experiments on various datasets have validated the effectiveness of HASFL, demonstrating at least 4$\times$ faster convergence and 1\% higher training and testing accuracy compared to state-of-the-art benchmarks.

\ifCLASSOPTIONcaptionsoff
  \newpage
\fi



%



\bibliographystyle{IEEEtran}
\bibliography{reference}
\end{document}